\documentclass{article}
\usepackage[preprint]{colm2026_conference}

\usepackage{microtype}
\usepackage{graphicx}
\usepackage{amsmath,amssymb,amsfonts,amsthm}
\usepackage{bm}
\usepackage{mathtools}
\usepackage{mathrsfs}
\usepackage{xparse}
\usepackage{booktabs}
\usepackage{url}
\usepackage{lineno}
\usepackage{xcolor}
\usepackage{colortbl}
\usepackage{multirow}
\usepackage{xspace}
\usepackage{siunitx}
\usepackage{hyperref}

\definecolor{darkblue}{rgb}{0,0,0.5}
\hypersetup{colorlinks=true,citecolor=darkblue,linkcolor=darkblue,urlcolor=darkblue}

\def\eqref#1{equation~\ref{#1}}
\def\1{\bm{1}}

\DeclareMathAlphabet{\mathsfit}{\encodingdefault}{\sfdefault}{m}{sl}
\SetMathAlphabet{\mathsfit}{bold}{\encodingdefault}{\sfdefault}{bx}{n}

\NewDocumentCommand{\var}{O{s} m O{}}{%
  \ensuremath{#1_{#2}^{#3}}% add \vphantom{<bizarre sup>}
}

\providecommand{\mainlabelprefix}{}
\newcommand{\mainref}[1]{\ref{\mainlabelprefix#1}}
\newcommand{\BigO}{\mathcal O}
\newcommand{\commentout}[1]{}

\definecolor{light-gray}{gray}{0.80}

\newtheorem{mytheorem}{Theorem}
\newtheorem{theorem}[mytheorem]{Theorem}

\newtheorem{assumption}{Assumption}
\newtheorem{lemma}{Lemma}
\theoremstyle{definition}

\newtheorem{proposition}{Proposition}
\newtheorem{remark}{Remark}
\theoremstyle{plain}
\newtheorem{remk}{Remark}

\newcommand\atsign{@}

\title{Compute Allocation in Bilinear Contrastive Representation Learning via Scaling Laws}

\author{Ziyan Chen \\
The University of Sydney \\
\texttt{ziyan.chen\atsign sydney.edu.au}
\And
Zhongzhu Zhou \\
The University of Sydney and Together AI \\
\texttt{zhongzhu.zhou\atsign sydney.edu.au}
\And
Dingxuan Zhou \\
The University of Sydney \\
\texttt{dingxuan.zhou\atsign sydney.edu.au}
}

\begin{document}
\ifcolmsubmission
\linenumbers
\fi

\maketitle

\begin{abstract}
Efficient reasoning systems increasingly rely on large contrastive representation models for multimodal retrieval, grounding, and decision-making. However, under a limited compute budget, it is unclear how to allocate resources among representation dimension, paired data, and optimization steps. We study this question in a sketched bilinear contrastive-learning model. Under a paired Gaussian latent-variable setup, we derive a risk decomposition into irreducible risk, approximation error, GD bias, GD variance, and a cross term. The resulting scaling law shows that contrastive learning differs from standard linear regression because the effective degrees of freedom are pairs of spectral directions rather than single directions. This product effective dimension determines how finite-sample noise grows with optimization time. Most importantly, using the compute proxy \(\mathcal C \asymp LNM^2\), we derive compute-allocation rules balancing model size, data, and training time. To the best of our knowledge, our result provides the first theoretical backbone for resource allocation in contrastive representation learning through provable scaling laws.
\end{abstract}
\section{Introduction}
\label{sec:introduction}

Modern reasoning systems often rely on learned representations before the final reasoning step. Multimodal, retrieval-based, and representation-based pipelines use embeddings to retrieve, ground, and compare evidence; contrastive representation learning is a central way to build such embeddings. It learns representations by pulling related pairs together and pushing unrelated pairs apart, with roots in metric learning~\citep{hadsell2006dimensionality}, CPC/InfoNCE-style objectives~\citep{oord2018representation}, and self-supervised visual learning~\citep{chen2020simple,he2020momentum,khosla2020supervised}. The same principle drives CLIP-style image--text pretraining, which supports zero-shot transfer, retrieval, and multimodal alignment~\citep{radford2021learning,jia2021scaling,zhai2022lit}. Thus, it is interesting to study compute-intensive upstream components for efficient reasoning systems.

This motivates a compute-allocation question: under a fixed compute budget, how should one divide resources among representation dimension/model size, paired data, and optimization steps? In our sketched contrastive setting these correspond to sketch dimension \(M\), sample size \(N\), and GD steps \(L\). Larger \(M\) improves representation capacity but increases per-step cost; larger \(N\) provides more paired information but increases full-batch cost; larger \(L\) reduces optimization bias but can fit finite-sample noise. A useful theory should therefore describe how approximation, optimization, and sampling trade off jointly.

Scaling laws provide the natural tool. Empirical scaling laws describe loss as a function of model size, data, and compute across language, vision, translation, speech, and multimodal modeling~\citep{kaplan2020scaling,hestness2017deep,henighan2020scaling,hoffmann2022training,zhai2022scaling,muennighoff2023scaling}, while theoretical work seeks to explain the exponents and risk components in tractable models~\citep{hutter2021learning,sharma2020neural,maloney2022solvable,bahri2024explaining,bordelon2024dynamical,atanasov2024scaling,paquette2024phases,dohmatob2024tale}. Recent sketched linear-regression analyses connect approximation, optimization, data, and compute under power-law spectra~\citep{lin2024scaling,lin2025datareuse}. However, they concern single-view prediction, whereas contrastive representation learning must align two views.

Contrastive learning is also less understood from this allocation perspective. Existing theory explains latent-factor recovery, augmentation effects, and differences from reconstruction-based learning~\citep{arora2019theoretical,wang2020understanding,tosh2021contrastive,zimmermann2021contrastive,haochen2021provable,ji2021power}, and recent work studies statistical consistency and generalization~\citep{li2026statistical}. Empirical studies of CLIP-like pretraining report predictable scaling across zero-shot classification, retrieval, probing, and fine-tuning~\citep{cherti2023reproducible,li2023inverse,nezhurina2025scaling}. What remains missing is a decomposition of contrastive risk into approximation, optimization, and sampling terms that yields explicit compute-allocation rules.

We develop such a theory in a sketched bilinear contrastive model. We study paired Gaussian views, a bilinear score on sketched inputs, and a negative-population contrastive objective~\citep{wang2020understanding} that reduces to a sketched Gaussian quadratic surrogate. The model is idealized rather than a quantitative description of CLIP-style scaling laws, but it preserves the key pairwise alignment structure. Unlike linear regression, whose degrees of freedom are individual spectral directions, bilinear contrastive learning activates spectral pairs. This product effective dimension controls how finite-sample noise grows with optimization time and changes the scaling law.

Our main contributions are as follows.
\begin{itemize}
    \item \textbf{A sketched bilinear contrastive learning setup.}
    We propose a Gaussian paired-view model for contrastive representation learning. The model is idealized enough to permit exact risk decompositions, while preserving the two-view alignment structure that makes contrastive learning different from single-view prediction.

    \item \textbf{Scaling laws for empirical GD.}
    We derive explicit scaling laws with respect to sketch dimension \(M\), sample size \(N\), and the effective optimization horizon \(L_{\mathrm{eff}}\gamma\) by bounding approximation error, GD bias, GD variance, and the cross term. A key quantity is the product effective dimension, which counts active pairs of spectral directions rather than active single directions.

    \item \textbf{Guidance for compute allocation.}
    Under the compute proxy \(\mathcal C \asymp LNM^2\), our final scaling law gives transparent rules for balancing representation dimension, paired data, and optimization time. This turns the risk decomposition into a resource-allocation guide for compute-efficient contrastive representation learning.
\end{itemize}

To the best of our knowledge, this is the first provable scaling-law analysis on population risk of a contrastive learning objective with an explicit risk decomposition. Its main message is that contrastive learning scales differently from linear regression because it learns interactions between pairs of spectral directions. Our work provides a starting point for further investigating its scaling behaviour under different setups, guiding contrastive representation training under limited compute.

\paragraph{Notation.}
For two positive quantities \(f\) and \(g\), we write \(f\lesssim g\)
(equivalently, \(f=\BigO(g)\)) and \(f\gtrsim g\) (equivalently,
\(f=\Omega(g)\)) if the inequality holds up to an absolute constant. We write
\(f\asymp g\) or \(f=\Theta(g)\) when both bounds hold. For matrices \(A\) and
\(B\) of compatible dimensions, \(\langle A,B\rangle:=\operatorname{tr}(A^\top
B)\) denotes the Frobenius inner product. We use \(\mathbb E_{\mathcal D}\) for expectation over the training sample, conditional on the sketch when the sketch is fixed.

\section{Preliminary}
\label{sec:preliminary}

\paragraph{Contrastive-learning background.}
Contrastive learning starts from paired views \(\{a_i\}_{i=1}^n,\{b_j\}_{j=1}^n\), where \(a_i\) is a positive view matched with \(b_j\) if \(i=j\), otherwise a negative view. For instance of image-caption matching, \((a_i, b_i)\) is a pair of matched image and caption. A common formulation trains encoders \(f_\theta\) and \(g_\theta\), which map to the same encoding dimension, through a similarity score normalized by \(\tau\):
\[
    s_\theta(u,v)
    :=
    \frac{\langle f_\theta(u),g_\theta(v)\rangle}{\tau},
    \quad u \in \{a_i\}_{i=1}^n,
    \quad v \in \{b_j\}_{j=1}^n.
\]
The goal is to learn representations that assign high scores to matched pairs and low scores to mismatched pairs, thereby recovering encoders \(f_\theta\) and \(g_\theta\), which are useful for downstream tasks, for example modal-feature extraction~\citep{oord2018representation,radford2021learning}. 
% Our sketched bilinear model can be viewed as replacing nonlinear encoders \(f_\theta\),\(g_\theta\) with a bilinear interaction matrix.

\subsection{Problem Setup}
\label{subsec:problem-setup}
Following the alignment idea in contrastive learning, we consider a sketched bilinear model, where we replace the nonlinear encoders \(f_\theta\),\(g_\theta\) with a bilinear interaction matrix. Figure~\ref{fig:problem-setup} gives a visual summary of the latent paired views, sketching map, bilinear score, and empirical GD procedure. We start with introducing the sketched data setup, followed by bilinear model structure, objective function, and recursive gradient descent.

\begin{figure*}[t]
    \centering
    \begin{minipage}{0.95\textwidth}
        \centering
        \includegraphics[width=\linewidth]{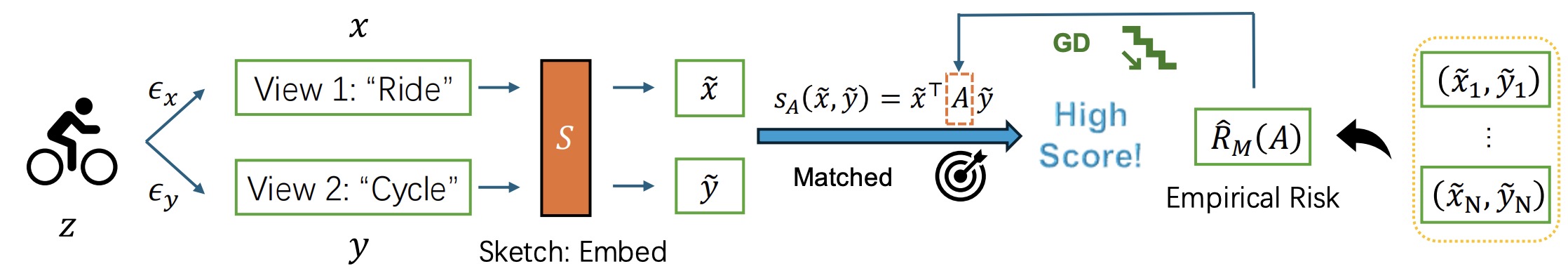}
        \small
        Starting from a shared latent variable \(z\), it generates two noisy matched views \(x\) and \(y\); the fixed sketch \(S\) maps them to \(\widetilde x\) and \(\widetilde y\). The sketched bilinear score \(s_A(\widetilde x,\widetilde y)=\widetilde x^\top A\widetilde y\) is then trained by empirical GD on paired sketched samples to minimize \(\widehat R_M(A)\).
    \end{minipage}
    \caption{Illustration of the sketched bilinear contrastive-learning setup.}
    \label{fig:problem-setup}
\end{figure*}

\paragraph{Latent-variable contrastive setup.}
Let \(D\in \mathbb N\cup\{\infty\}\) be the ambient dimension. Consider
\[
    z\sim \mathcal N(0,\Lambda_z),\qquad
    \epsilon_x,\epsilon_y \overset{\mathrm{i.i.d.}}{\sim}
    \mathcal N(0,\Lambda_\epsilon),
\]
with $z \in \mathbb{R}$ the latent variable, and \(\epsilon_x,\epsilon_y \in \mathbb{R}\) two independent perturbations. They generate two matched views based on $z$:
\[
    x=z+\epsilon_x,\qquad y=z+\epsilon_y .
\]
The two views share a latent component: for example, two texts $x$ and $y$ describing the same event $z$, but with independent view-specific noise. Define the marginal covariance and the cross-covariance by
\begin{align*}
    H:=\mathbb E[xx^\top]
      =\mathbb E[yy^\top]
      =\Lambda_z+\Lambda_\epsilon, \quad
    C:=\mathbb E[xy^\top]=\Lambda_z.
\end{align*}

Let \(S\in\mathbb R^{M\times D}\) be a fixed full-row-rank sketching matrix, where \(M<D\) represents the model dimension (as well as input data dimension). Its spectral alignment properties are specified in Assumption~\ref{ass:contrastive-source}.
For a paired input \((x,y)\), the learner observes only the sketched pair $(\widetilde x=Sx,\ \widetilde y=Sy)$, which simulates limitations in observation. For the example of two texts describing the same event, this sketch can be viewed as an embedding-dimension bottleneck: a finite-dimensional token representation cannot describe the text from all perspectives.

\paragraph{Bilinear contrastive score.}
For an unsketched setup, we consider a bilinear score function with interaction matrix \(W\) between paired observations,
\[
    s_W(x,y):=x^\top Wy,
    \qquad W\in\mathbb R^{D\times D}.
\]
Let \(p_{+}\) be the joint distribution of a positive pair \((x,y)\), with marginals \(p_x\) and \(p_y\).  A large-negative-sample objective can be written as \citep{wang2020understanding}
\[
    \mathcal L(s)
    :={}
    \mathbb E_{(x,y)\sim p_{+}}[\log
    \mathbb E_{y'\sim p_y}
    \exp\left(\frac{s(x,y')-s(x,y)}{\tau}\right)],
\]
with \(\tau\) the temperature. For the score \(s_W(x,y)=x^\top Wy\), the independent negative \(y'\sim\mathcal N(0,H)\) makes \(x^\top Wy'\) Gaussian with conditional variance \(x^\top WHW^\top x\). Hence
\[
\begin{aligned}
    \mathcal L(W)
    &={}
    -\frac1\tau\langle W,C\rangle
    +
    \mathbb E_x
    \left[
        \log
        \mathbb E_{y'}
        \exp\left(\frac{x^\top Wy'}{\tau}\right)
    \right] \\
    &={}
    -\frac1\tau\langle W,C\rangle
    +
    \frac{1}{2\tau^2}
    \operatorname{tr}(W^\top HWH).
\end{aligned}
\]
Since the temperature only rescales \(W\), replacing \(W/\tau\) by \(W\) gives the quadratic objective:
\[
    R(W)
    :=
    -\langle W,C\rangle
    +\frac12 \operatorname{tr}(W^\top HWH),
    \qquad W\in\mathbb R^{D\times D}.
\]
Here the cross-covariance term \(-\langle W,C\rangle\) represents positive-pair
alignment, while the marginal term \(\operatorname{tr}(W^\top HWH)\) comes from the Gaussian moment-generating function for independent negatives. Its full population minimizer is
\[
    W^\star=H^{-1}CH^{-1}.
\]

For the sketched model \(W=S^\top A S\), the learnable score matrix is \(A \in \mathbb{R}^{M\times M}\). Here we define the sketched marginal covariance and sketched cross covariance as
\[
    \Sigma:=SHS^\top,
    \qquad
    C_M:=SCS^\top.
\]
Then the risk of the sketched-space matrix parameter is
\[
    R_M(A) 
    :=R(S^\top A S)
    =
    -\langle A,C_M\rangle
    +\frac12\operatorname{tr}(A^\top \Sigma A\Sigma).
\]
Similarly, its minimizer can be shown to be
\[
    A^\star
    =
    \Sigma^{-1}C_M\Sigma^{-1}
    =
    (SHS^\top)^{-1}SCS^\top(SHS^\top)^{-1}.
\]

\paragraph{Full-batch Empirical Gradient Descent.}

Under the population risk, the excess risk of an estimated parameter \(A\), describing the quality difference between the estimated parameter and the best possible parameter, is defined as
\[
    R_M(A)-R_M(A^\star)
    =
    \frac12\|A-A^\star\|_{\Sigma,\Sigma}^2.
\]
Given \(N\) pairs of sketched sample \((\widetilde x_i,\widetilde y_i)_{i=1}^N\), define their empirical sketched marginal and cross covariance as
\[
    \widehat\Sigma_x
    :=
    \frac1N\sum_{i=1}^N\widetilde x_i\widetilde x_i^\top,
    \
    \widehat\Sigma_y
    :=
    \frac1N\sum_{i=1}^N\widetilde y_i\widetilde y_i^\top,
    \
    \widehat C
    :=
    \frac1N\sum_{i=1}^N\widetilde x_i\widetilde y_i^\top.
\]
% and empirical sketched cross covariance as
% \[
%     \widehat C
%     :=
%     \frac1N\sum_{i=1}^N\widetilde x_i\widetilde y_i^\top.
% \]
We consider full-batch empirical GD on empirical loss
\[
    \widehat R_M(A)
    =
    -\langle A,\widehat C\rangle
    +
    \frac12\operatorname{tr}
    \left(
        A^\top\widehat\Sigma_xA\widehat\Sigma_y
    \right).
\]
Consequently, the GD recursion is defined as
\[
A_t
=
A_{t-1}
-
\gamma_t
\left(
\widehat\Sigma_xA_{t-1}\widehat\Sigma_y-\widehat C
\right),
\qquad
t=1,\ldots,L,
\]
with initialization \(A_0=0\), where \(\gamma_t\) is the learning rate at step $t$ and \(L\) is the number of full-batch GD steps. Here we use a geometrically decaying stepsize schedule for the convergence of the optimization procedure
\[
    \gamma_t=\frac{\gamma}{2^\ell},
    \qquad
    \ell=\left\lfloor \frac{t}{L_{\mathrm{eff}}}\right\rfloor,
    \qquad
    t=1,\ldots,L,
\]
where we define the effective steps as $L_{\mathrm{eff}}:=\left\lfloor \frac{L}{\log L}\right\rfloor .$

\subsection{Assumptions}
\label{subsec:assumptions}

% ============================================================
% Assumptions for the contrastive approximation scaling
% ============================================================
\begin{assumption}[Aligned power-law spectra]
\label{ass:aligned-power-laws}
The covariance operators \(\Lambda_z\) and \(\Lambda_\epsilon\) are simultaneously diagonalizable: there exists an orthonormal basis \(\{u_i\}_{i=1}^D\) such that
\[
    \Lambda_z=
    \sum_{i=1}^D\lambda_{z,i}u_iu_i^\top,
    \qquad
    \Lambda_\epsilon=
    \sum_{i=1}^D\lambda_{\epsilon,i}u_iu_i^\top .
\]
Also, there exist constants \(a,b>1\) with \(a>b+\frac12\), and constants \(0<c_0 \le c_1\) such that
\[
    c_0 i^{-a}\le \lambda_{z,i} \le c_1 i^{-a},
    \qquad
    c_0 i^{-b}\le \lambda_{\epsilon,i} \le c_1 i^{-b},
    \qquad i\ge 1.
\]
\end{assumption}

This assumption fixes a common spectral coordinate system for the latent variable and perturbation. Power-law spectra of this form are widely used in sketched linear-regression and kernel-learning scaling analyses
\citep{caponnetto2007optimal,lin2017distributed,bahri2024explaining,lin2024scaling,lin2025datareuse}. The power-law decay models the fact that most signal and noise energy is concentrated in leading spectral directions.

\begin{assumption}[GD stepsize schedule and effective horizon]
\label{ass:gd-schedule}
The initial stepsize \(\gamma>0\) is deterministic and satisfies
\(\gamma\le c_\gamma\) for a sufficiently small constant
\(c_\gamma>0\). Also, the effective horizon is non-degenerate, namely
\(R:=L_{\mathrm{eff}}\gamma\gtrsim 1.\)
\end{assumption}

This assumption ensures that the full-batch GD dynamics are stable, since it keeps the empirical risk contraction of GD, while still running for a meaningful amount of optimization time. The condition rules out the degenerate regime in which the effective optimization horizon vanishes. Without this condition, the algorithm would have essentially no opportunity to learn the sketched population minimizer.

\begin{assumption}[Spectral truncation sketch and contrastive source condition]
\label{ass:contrastive-source}
Let \(\delta:=a-b\). We assume that the fixed sketch preserves the
leading population spectral order in the sketched target. Specifically, there
exists an ordered orthonormal basis \(\{v_i\}_{i=1}^M\) of the sketched space,
indexed by the first \(M\) population eigendirections, such that
\[
\begin{aligned}
    \Sigma=
    \sum_{i=1}^M\mu_i v_iv_i^\top,
    \qquad
    C_M=
    \sum_{i=1}^M\kappa_i\mu_i v_iv_i^\top .
\end{aligned}
\]
Here the eigenvalues and source coefficients obey
\[
    c_\mu i^{-b}\le \mu_i\le C_\mu i^{-b}, \quad
    \kappa_i
    =\frac{\lambda_{z,i}}{\lambda_{z,i}+\lambda_{\epsilon,i}}\asymp i^{-\delta},
\]
where \(i=1,\ldots,M\), with constants independent of \(M\). In addition, all population eigendirections outside the order \(M\) are discarded by the sketch:
\[
    S u_i=0,
    \qquad i>M .
\]
\end{assumption}

Under this condition, the first \(M\) spectral directions of \(H\) and \(C\), along with their order, are preserved after sketching into \(\Sigma\) and \(C_M\), while the remaining population directions belong to the null space of \(S\). Source conditions and spectral truncation assumptions are widely used in
linear-regression scaling analyses~\citep{lin2024scaling,lin2025datareuse}.
The condition is imposed directly on the sketched covariance pair
\((\Sigma,C_M)\), and it is not automatic for an arbitrary sketch: even when \(H\) and \(C\) commute, the matrices \(SHS^\top\) and \(SCS^\top\) need not commute. In normal linear regression setup in form of \(y=x^\top w^*+\xi\), the target \(w^*\) is not automatically dependent on the distribution of \(x\). In comparison, the target depends on the sketched covariance in our contrastive setup, since
\(A^\star=\Sigma^{-1}C_M\Sigma^{-1}\). The condition therefore makes explicit that the sketched target inherits the same ordered source alignment as the
population target.

\section{Main Results}
\label{sec:main-results}

\subsection{Scaling Laws Backbone}
\label{subsec:scaling-behaviour}
We now state the main scaling law for empirical gradient descent in the sketched contrastive learning problem. Following the risk-decomposition idea of~\citet{lin2025datareuse}, we separate the population risk in the following way.

\begin{proposition}[Informal risk decomposition]
\label{prop:main-risk-decomposition}
Let \(A_L\) be the \(L\)-step empirical GD iterate. Then the population risk admits the schematic decomposition
\[
    \mathbb E_{\mathcal D}
    \left[
        R_M(A_L)
    \right]
    =
    \operatorname{Irreducible}
    +
    \operatorname{Approx}
    +
    \operatorname{Bias}
    +
    \operatorname{Var}
    +
    \operatorname{Cross}.
\]
Here the irreducible term is the best value achieved by the ground-truth unsketched contrastive predictor, while the approximation term is the error caused by using an \(M\)-dimensional representation. The bias, variance, and cross terms come from the sketched excess risk \(R_M(A_L)-R_M(A^\star)\): bias is caused by incomplete training, variance by finite-sample noise learned during optimization, and the cross term by their interaction.
\end{proposition}

By Cauchy--Schwarz, the cross term is bounded and an upper-bound analysis reduces to
\[
    \mathbb E_{\mathcal D}
    \left[
        R_M(A_L)
    \right]
    \lesssim
    \operatorname{Irreducible}
    +
    \operatorname{Approx}
    +
    \operatorname{Bias}
    +
    \operatorname{Var}.
\]
% This decomposition already identifies the compute-allocation roles of the terms: approximation is governed by representation dimension \(M\), because larger sketches provide more representation capacity; GD bias is governed by the effective optimization horizon \(R\), because longer training learns more signal; and GD variance is governed by sample size \(N\) together with the number of active directions learned by optimization, because more paired data dilutes empirical noise while longer training can fit more of it.

We next combine the approximation, GD-bias, and GD-variance estimates. Recall the effective horizon \(R:=L_{\mathrm{eff}}\gamma\) and effective decay power \(\delta:=a-b\).
% The theorem is stated away from the boundary
% \(\delta=(b+1)/2\), so that the approximation term has a power-law form.
We propose the following theorem of contrastive scaling law.

\begin{theorem}[Scaling law for empirical GD in sketched contrastive learning]
\label{thm:main-gd-scaling}
Suppose Assumptions~\ref{ass:aligned-power-laws}, \ref{ass:gd-schedule}, and \ref{ass:contrastive-source} hold, with \(\frac12<\delta<b+\frac12\). Assume the effective horizon satisfies
\[
    1\lesssim R\lesssim \min\left\{\frac{N}{M},\,M^{2b}\right\}.
\]
Then with probability at least \(1-\exp(-\Omega(M))\),
\begin{align*}
    \mathbb E_{\mathcal D}[R_M(A_L)]
    &={}
    \underbrace{R_{\mathrm{irr}}}_{\mathclap{\text{irreducible}}}
    +
    \underbrace{
    \Theta\left(M^{1-2\delta}\right)
    }_{\mathclap{\text{approximation}}}
    +
    \underbrace{
    \Theta\left(R^{\frac{1-2\delta}{2b}}\right)
    }_{\mathclap{\text{GD bias}}}
    \\
    &\quad+
    \underbrace{
    \Theta\left(
    \frac{D_{\times}(R,M)}{N}
    \right)
    }_{\mathclap{\text{GD variance}}}
    +
    \underbrace{
    \BigO\left(
    R^{\frac{1-2\delta}{4b}}
    \sqrt{D_{\times}(R,M)/N}
    \right)
    }_{\mathclap{\text{cross term}}}.
\end{align*}
Here
\[
    D_{\times}(R,M)
    =
    \begin{cases}
    R^{1/b}\log(eR), & 1\le R^{1/b}\le M, \\
    R^{1/b}\left(1+\log\dfrac{M^2}{R^{1/b}}\right),
        & M<R^{1/b}<M^2, \\
    M^2, & R^{1/b}\ge M^2.
    \end{cases}
\]
The hidden constants in \(\Theta(\cdot)\) and \(\BigO(\cdot)\) depend only on exponents and constants in assumptions.
\end{theorem}

\paragraph{Interpretation of the scaling law.}
The theorem makes the allocation roles explicit. The approximation term scales with \(M\), since a larger representation dimension preserves more information and reduces sketching error. The GD-bias scales with the effective horizon \(R=L_{\mathrm{eff}}\gamma\), since longer optimization reduces incomplete training. The GD-variance scales as \(D_{\times}(R,M)/N\): more paired samples reduce variance through \(1/N\), while longer optimization introduces more noise. The cross term is controlled by bias and variance, so it is not a separate scaling bottleneck.

The horizon conditions in Theorem~\ref{thm:main-gd-scaling} have two roles: \(R\lesssim M^{2b}\) keeps the GD-bias term in an unsaturated (not fully
learned) regime, while \(R\lesssim N/M\) ensures the empirical marginal
covariances can be replaced, with small perturbation, by the population covariance, with high probability (covariance constraint). Previous sketched linear-regression scaling-law papers propose covariance replacement lemmas for this purpose~\citep{lin2024scaling,lin2025datareuse}, but our setup is different since the bilinear setup requires uniform control over view--view interactions.

\paragraph{Proof sketch.}
For the approximation term, the contrastive source condition is imposed directly on the sketched covariance pair \((\Sigma,C_M)\) generated by the fixed
full-row-rank sketch. It requires the sketched whitened contrastive signal to retain an ordered source profile \(\kappa_i\asymp i^{-\delta}\) over its \(M\) effective directions. Thus the part not represented by the sketch is controlled by the remaining source tail, which gives the rate
\(M^{1-2\delta}\).

For the GD-bias term, one can think of gradient descent as learning coordinates
from easiest to hardest. Here we use curvature to denote the steepness of the risk along a parameter direction. Specifically, under view-view setup, the Hessian
maps a perturbation \(B\) to \(\Sigma B\Sigma\). Hence, for
\(\Sigma v_i=\mu_i v_i\), the matrix direction \(v_iv_j^\top\) has curvature
\(\mu_i\mu_j\). Large-curvature directions have strong gradient
signals, so they are learned well during optimization, while low-curvature
directions are learned slowly. Also, the source condition shows the amount of target signal in each direction. For the upper bound, we split the target into a learnable head, whose residual is reduced by the GD filter, and a high-frequency tail, whose remaining mass is small. The lower bound keeps the part
of the target that lies in low-curvature directions, showing that finite-time GD
cannot learn that part faster.

The GD-variance term is complementary. In our bilinear contrastive problem, noise is indexed by pairs of directions \((i,j)\), where a pair becomes active when its product curvature \(\mu_i\mu_j\) is large enough to be learned. This leads to the product effective
dimension
\[
    d_{\mathrm{prod}}(R,M)
    :=
    \sum_{i,j\le M}
    \min\{1,(R\mu_i\mu_j)^2\},
\]
whose power-law scale is \(D_{\times}(R,M)\). The lower bound shows that many of
these active pairs carry empirical fluctuations, giving the same rate as the upper bound.

\subsection{Compute Budget Allocation}
\label{subsec:compute}

\paragraph{Compute proxy.}
Scaling-law studies often express training compute as the product of model size, number of processed examples or tokens, and optimization steps, up to hardware- and architecture-dependent constants \citep{kaplan2020scaling,
hoffmann2022training}. In our sketched bilinear model,
% , the parameter matrix has
% \(M^2\) entries and each full-batch GD step processes \(N\) paired examples. 
we use the leading-order compute proxy
\[
    \mathcal C
    \asymp
    LNM^2.
\]
Here \(L\) counts full-batch GD steps, \(N\) is the number of paired training
examples processed at each step, and \(M^2\) is the number of entries in the
sketched bilinear parameter \(A\in\mathbb R^{M\times M}\). Thus
\(LNM^2\) should be read as a leading-order operation-count proxy: it ignores
hardware constants, logarithmic factors, and implementation details, but keeps
the dominant multiplicative dependence on optimization time, data size, and
model size. This product-form compute proxy is analogous to the training-compute
surrogates used in empirical scaling-law studies, where compute is modeled as
model size times the amount of processed data, and to recent sketched linear
regression scaling analyses that balance approximation, data, and optimization
under a simplified compute budget~\citep{kaplan2020scaling,hestness2017deep,
henighan2020scaling,hoffmann2022training,lin2024scaling,lin2025datareuse}.

\paragraph{Compute allocation.}
The effective optimization horizon is \(R=L_{\mathrm{eff}}\gamma\), and we use
the heuristic relation \(R\approx L\gamma/\log L\). Thus, when \(\gamma\) is kept
near a constant stable value, \(L\) and \(R\) are interchangeable up to logarithmic
factors. The question is how to allocate a fixed compute budget \(\mathcal C\)
among optimization time \(L\), sketch/model size \(M\), and sample size \(N\).

\begin{proposition}[Compute allocation heuristic]
\label{prop:compute-allocation}
Ignore logarithmic factors and set \(\gamma\asymp1\). Let
\(q:=2\delta-1\), \(0<q<2b\), and impose the effective-horizon constraint from Theorem~\ref{thm:main-gd-scaling},
\[
    1\lesssim R\lesssim \min\left\{\frac{N}{M},\,M^{2b}\right\}.
\]
Under the compute proxy \(\mathcal C\asymp LNM^2\), the leading allocation rules are as follows, up to logarithmic factors. The last two cases balance the leading excess-risk proxy, while the first records the cost of additionally forcing the product-unsaturated boundary.
\begin{enumerate}
    \item \textbf{Product-unsaturated, forced.} If \(M\asymp R^{1/b}\), then
    \[
        R\asymp\mathcal C^{\frac{b}{2b+3}},\qquad
        M\asymp\mathcal C^{\frac{1}{2b+3}},\qquad
        N\asymp\mathcal C^{\frac{b+1}{2b+3}}.
    \]
    \item \textbf{Covariance-limited regime.} If \(0<q<2b-1\), then
    \[
        R\asymp\mathcal C^{\frac{2b}{4b+3}},\qquad
        M\asymp\mathcal C^{\frac{1}{4b+3}},\qquad
        N\asymp\mathcal C^{\frac{2b+1}{4b+3}}.
    \]
    \item \textbf{Variance-balanced regime.} If \(2b-1<q<2b\), then
    \[
        R\asymp\mathcal C^{\frac{2b}{2b+4+q}},\qquad
        M\asymp\mathcal C^{\frac{1}{2b+4+q}},\qquad
        N\asymp\mathcal C^{\frac{q+2}{2b+4+q}}.
    \]
\end{enumerate}
The boundary case \(q=2b-1\) gives the same exponents from the covariance-limited and variance-balanced formulas up to logarithmic factors. Under these allocations, the leading expected loss is
\[
    \mathbb E_{\mathcal D}[R_M(A_L)]
    =
    R_{\mathrm{irr}}
    +
    \begin{cases}
        \Theta\left(\mathcal C^{-\frac{q}{4b+6}}\right),
            & \text{product-unsaturated, forced}, \\
        \Theta\left(\mathcal C^{-\frac{q}{4b+3}}\right),
            & \text{covariance-limited }(0<q<2b-1), \\
        \Theta\left(\mathcal C^{-\frac{q}{2b+4+q}}\right),
            & \text{variance-balanced }(2b-1<q<2b).
    \end{cases}
\]
In the variance-balanced regime, \(N/(RM)\asymp R^{(q-(2b-1))/(2b)}\), so the covariance constraint is slack except at the boundary.
\end{proposition}

\paragraph{Choosing time steps and learning rate.}
Let \(R_\star\) denote the target horizon prescribed by Proposition~\ref{prop:compute-allocation}. The learning-rate choice is a deterministic stable step, \(\gamma_\star \asymp c_\gamma,\) up to absolute constant factors, where \(c_\gamma\) is chosen small enough for the GD filters to be contractions on the covariance event and spectral alignment condition. We then choose \(L_{\mathrm{eff}}=\left\lceil \frac{R_\star}{\gamma_\star}\right\rceil.\)
The number of GD steps \(L\) is then the smallest integer satisfying
\(\lfloor L/\log L\rfloor\ge L_{\mathrm{eff}}\). Thus, when \(\gamma_\star\asymp1\), the compute-allocation exponents can be read as explicit prescriptions for the number of GD steps: realize the target \(R_\star\) in Proposition~\ref{prop:compute-allocation} by taking
\(L\asymp R_\star\log R_\star\), with \(M\) and \(N\) chosen according to the corresponding row above.

\paragraph{Allocation intuition.}
The allocation rule should be interpreted in the theorem regime
\[
    1\lesssim R\lesssim \min\{N/M,M^{2b}\},
\]
where the sketch satisfies the spectral alignment/source condition and GD remains bias-unsaturated. The constraint \(R\lesssim N/M\) says that data reuse is limited: increasing the effective horizon requires enough paired samples for empirical covariance replacement.

In the forced product-unsaturated case and in the covariance-limited regime \(0<q<2b-1\), the leading allocation keeps \(N\) at the lower boundary \(N\asymp RM\). This spends just enough data to stabilize the empirical covariance terms, and allocates the remaining compute to balance representation dimension \(M\) against optimization horizon \(R\).

When \(2b-1<q<2b\), the covariance boundary would leave the GD-variance term too large. The variance-balanced allocation therefore uses extra data,
\[
    N\asymp R^{q/(2b)+1/b},
\]
so that approximation, GD bias, and GD variance are all of comparable order. Intuitively, smoother sources make approximation and bias decay faster, so finite-sample noise becomes the term that must receive additional compute.

Overall, under a full-row-rank sketch satisfying Assumption~\ref{ass:contrastive-source}, compute should first respect the covariance-control requirement, then be split between model dimension, data, and optimization time according to whether variance is already lower order or must be balanced.

\section{Experiments}
\label{sec:experiments}

\begin{figure*}[!t]
    \centering

    \begin{minipage}{0.48\textwidth}
        \centering
        \includegraphics[width=\linewidth]{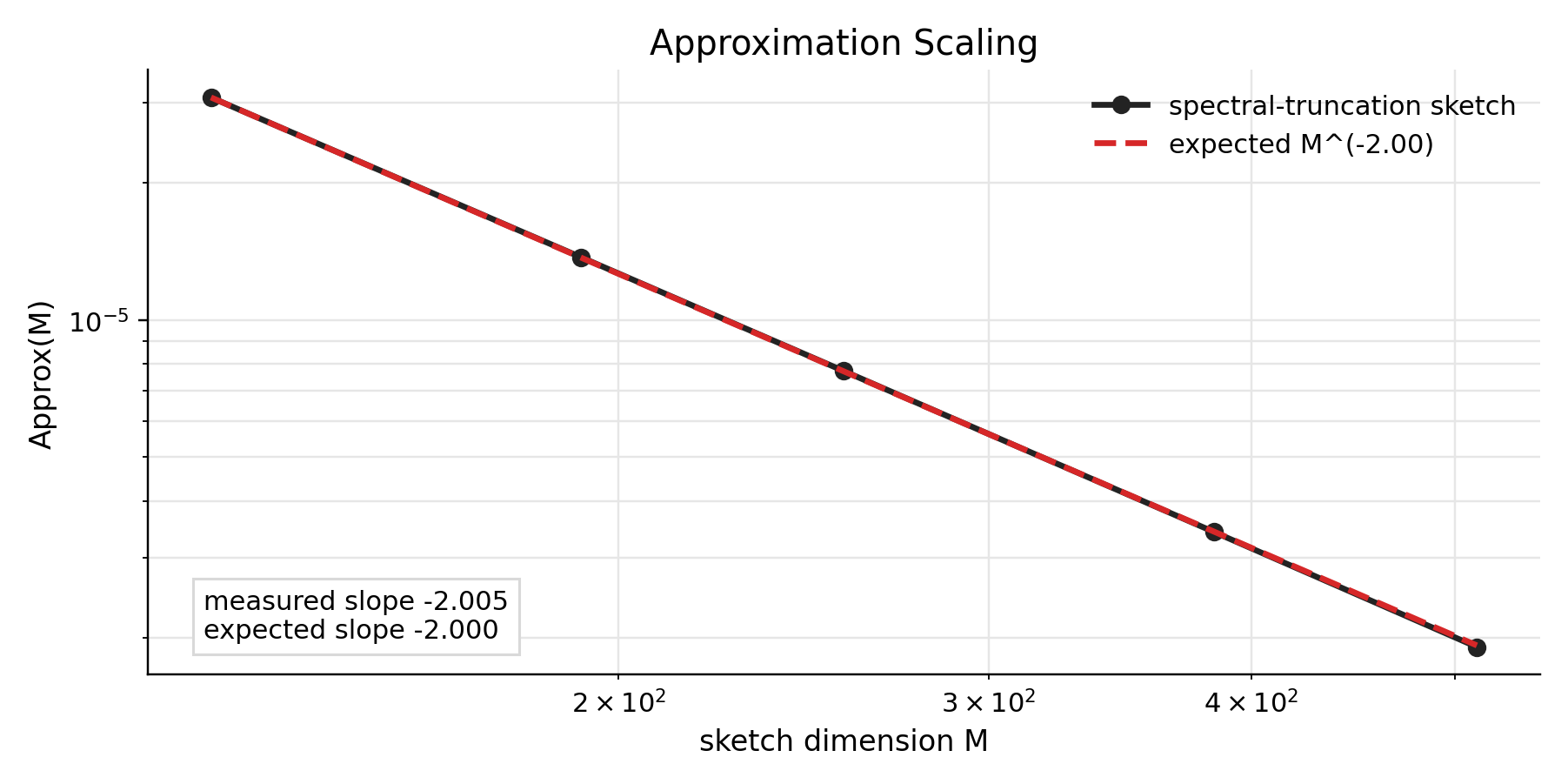}
    \end{minipage}
    \hfill
    \begin{minipage}{0.48\textwidth}
        \centering
        \includegraphics[width=\linewidth]{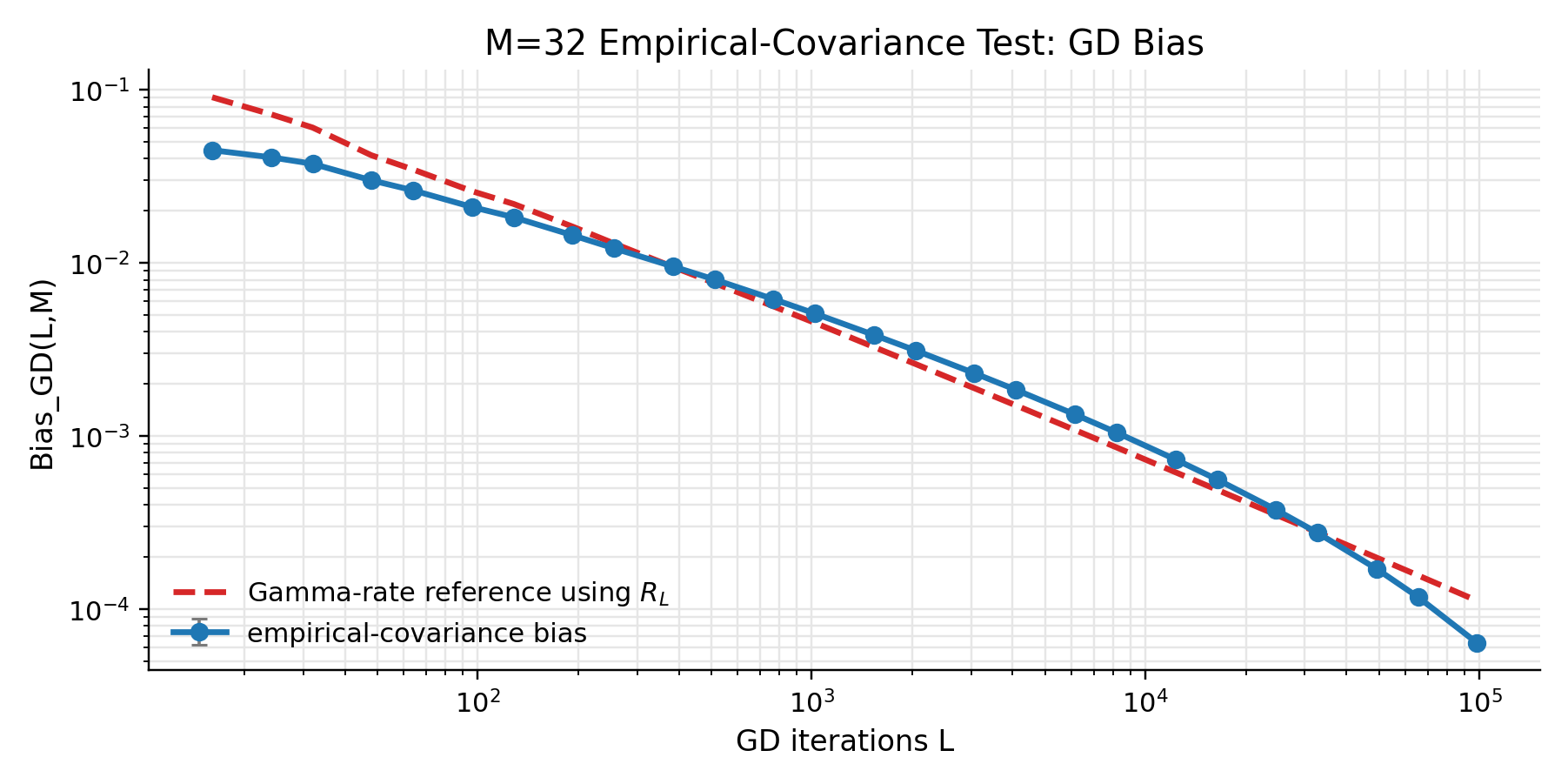}
    \end{minipage}

    \vspace{0.5em}

    \begin{minipage}{0.48\textwidth}
        \centering
        \includegraphics[width=\linewidth]{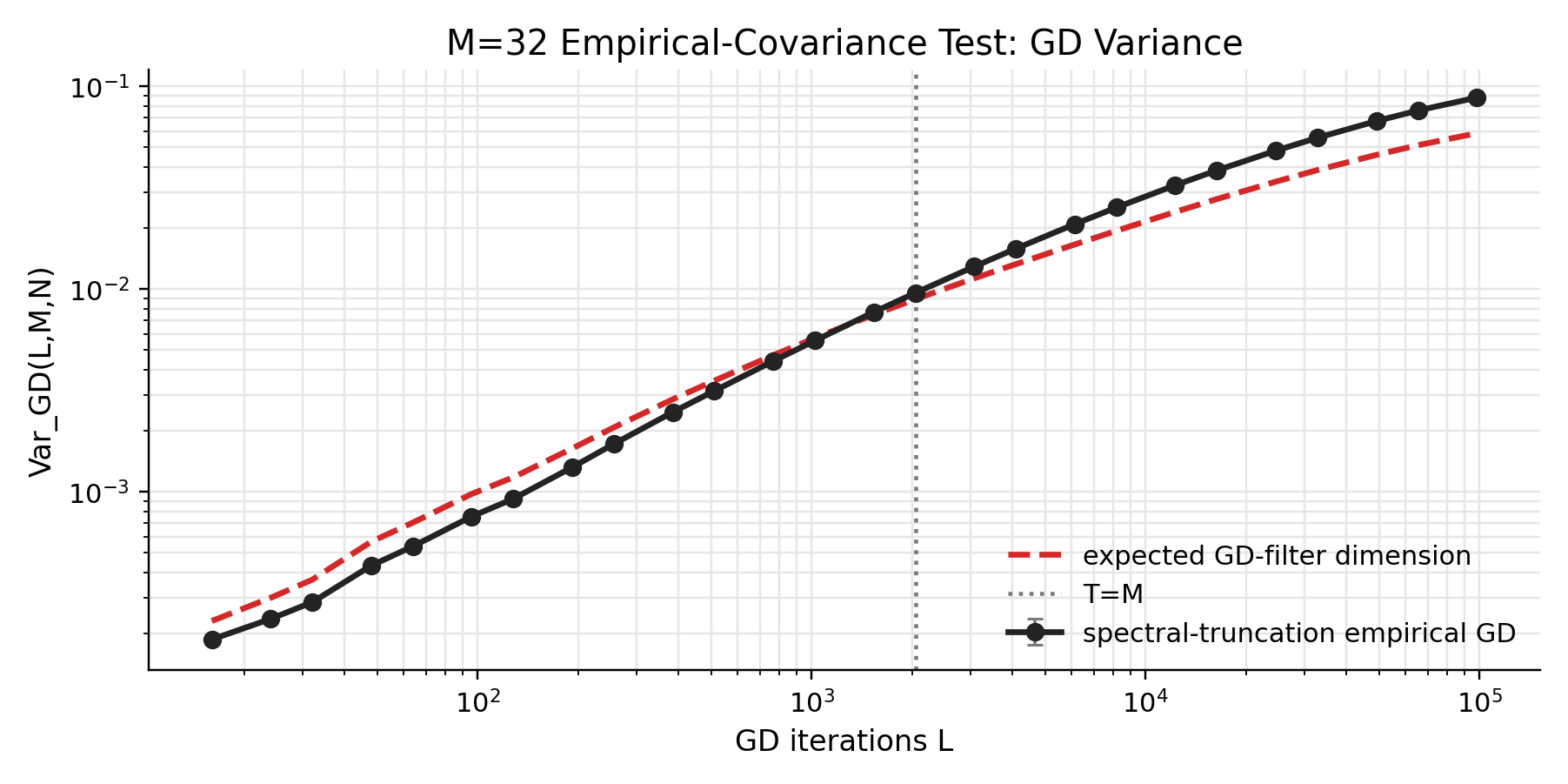}
    \end{minipage}
    \hfill
    \begin{minipage}{0.48\textwidth}
        \centering
        \includegraphics[width=\linewidth]{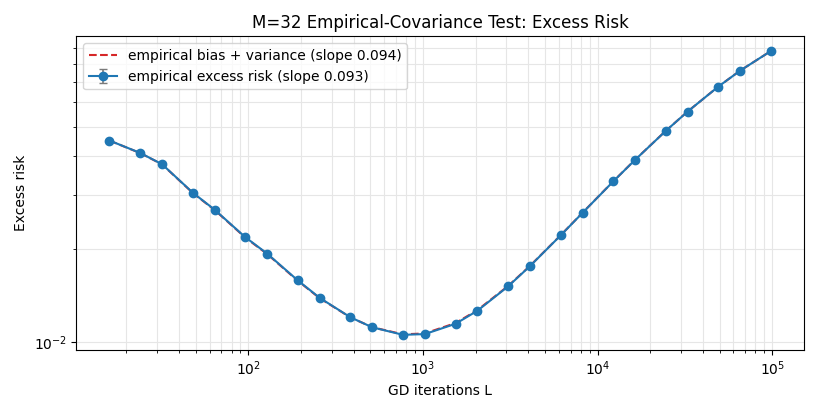}
    \end{minipage}

    \caption{
    Synthetic verification of the sketched contrastive scaling law. Top-left:
    approximation error versus sketch dimension \(M\). Top-right: GD bias versus
    the number of GD iterations \(L\), compared with the bias
    reference. Bottom-left: GD empirical variance versus \(L\) compared with the variance reference, with vertical markers indicating variance-regime transitions.
    Bottom-right: empirical excess risk along GD compared with the empirical
    bias-plus-variance.
    }
    \label{fig:synthetic-experiments}
\end{figure*}

% \subsection*{Synthetic setup}
% \label{subsec:synthetic-experiments}
\paragraph{Experiment Setups.}
We run synthetic experiments under the paired Gaussian latent-variable model of Theorem~\ref{thm:main-gd-scaling}. The two views share a latent component, have independent Gaussian noises with shared power-law eigenspectra, and are observed through a fixed full-row-rank sketch; the sketched bilinear model is then trained by full-batch GD. For the optimization panels, we enforce the source alignment and control the empirical marginal covariances, so that the experiments focus on testing the predicted GD bias/variance filters.

\paragraph{Empirical Validations.}
% \subsection*{Empirical verification}
% \label{subsec:empirical-verification}
Figure~\ref{fig:synthetic-experiments} summarizes the results. The approximation experiment varies the sketch dimension \(M\) and compares the empirical approximation error with the predicted power-law rate. The GD-bias experiment fixes \(M\) and varies the number of GD steps over a long horizon, comparing the measured GD bias with the rate predicted by Theorem~\ref{thm:main-gd-scaling}. The GD-variance experiment fixes \(M\) and \(N\), varies \(L\), and estimates the variance component across independently sampled datasets. The variance reference tracks the predicted finite-filter growth and bending across the tested horizon. Finally, the excess-risk experiment varies \(L\) and compares the empirical excess risk with the measured bias-plus-variance curve.

\paragraph{Interpretations.}
% \subsection*{Results}
% \label{subsec:results}
Overall, the synthetic results support the qualitative predictions of the theory. The approximation curve closely follows the predicted sketch-dimension exponent: the fitted slope is about \(-2.01\), compared with the predicted slope \(-2.00\). The GD-bias curve decreases monotonically with the optimization horizon; its fitted slope against \(\Gamma_L\) is about \(-0.86\), close to the theorem-guided slope \(-0.91\). The GD-variance curve grows as additional direction-pairs are learned from finite-sample fluctuations. It initially grows nearly linearly on the log--log scale and then bends in the intermediate regime; over the tested horizon it has not yet reached a flat saturated plateau. The excess-risk curve is U-shaped: early iterations are bias dominated, later iterations are variance dominated, and the minimum occurs near the point where the two terms are comparable. The empirical excess risk remains almost indistinguishable from the bias-plus-variance curve, with maximum relative discrepancy below \(0.6\%\), supporting the risk decomposition in Proposition~\ref{prop:main-risk-decomposition} and the treatment of the cross term as lower order.

\section{Inspirations and Limitations}
\label{sec:limitations}

There are three nearby lines of work. CLIP learns
image--text alignment by training modality-specific encoders on paired
image--caption data with a contrastive objective~\citep{radford2021learning}.
InfoNCE formalizes contrastive learning by comparing positive pairs against
negative samples~\citep{oord2018representation}. Recent data-reuse scaling laws study how optimization time and repeated use of finite data affect sketched linear regression~\citep{lin2025datareuse}. 

These works inspire different parts of our formulation. We imitate the spectral power-law assumptions and GD-based population-risk decomposition from sketched linear-regression scaling theory. CLIP motivates using a paired score function to evaluate the relationship between two views, while InfoNCE motivates treating alignment between matched views as both the training objective and the population-risk evaluation criterion. 

Compared with these works, we acknowledge the following limitations in our setup:

% Since we aim for a clean compute-allocation heuristic for contrastive learning, which requires a tractable model with explicit scaling behaviour, we acknowledge the following limitations relative to these works. These choices make the scaling law explicit but restrict the scope of the model:
\begin{itemize}
    \item \textbf{Bilinear matched-view representation.}
    We use a bilinear score, and assume the two views \(x\) and \(y\) have matched dimensions and are sketched by the same linear map. This abstracts away the nonlinear, modality-specific image and text encoders used in CLIP, where the raw inputs live in different spaces before being mapped to a common embedding space.

    \item \textbf{No explicit sampled negatives.}
    Our empirical training data consists of paired positive samples, with negative-sample effect represented by a large-negative-sample objective rather than explicit negative pairs. This is less general than InfoNCE training, where positives are explicitly contrasted against negative samples.

    \item \textbf{Simplified optimization and covariance control.}
    We analyze Gaussian data with full-batch GD under a covariance-concentration regime, where empirical marginal covariances are controlled by their population counterparts. This assumption is stronger than what is typically available in practical contrastive training and omits features such as mini-batch stochasticity.
\end{itemize}

Despite these simplifications, the result is still informative. From the scaling-law perspective, it extends linear-regression theory to a contrastive objective whose effective degrees of freedom come from products of view-specific spectral directions. From the contrastive-learning perspective, it complements existing approximation and generalization analyses by adding an explicit optimization scheme and showing how approximation, optimization bias, and finite-sample variance interact during training.

% Thus the theorem should be interpreted as a tractable spectral model for
% contrastive scaling, rather than a direct quantitative prediction for CLIP loss curves.
\section{Conclusion}
\label{sec:conclusion}

We studied compute allocation for sketched contrastive representation learning. In a paired Gaussian latent-variable model, we analyzed a quadratic Gaussian-negative contrastive objective trained by empirical GD. The resulting risk decomposition separates irreducible risk, sketching approximation, optimization bias, finite-sample variance and a cross term.

The main message is that contrastive representation learning has a different scaling structure from sketched linear regression because it learns interactions between two views. This creates a product effective dimension over spectral pairs, which determines how optimization time activates finite-sample noise. Combining this scaling law with the compute proxy \(\mathcal C\asymp LNM^2\) gives explicit guidance for allocating compute among representation dimension, paired data, and optimization steps: in covariance-limited regimes, data should be just large enough for covariance control, whereas in variance-balanced regimes, extra paired data is needed to keep finite-sample noise from dominating.

Several directions remain open. It would be valuable to sharpen the covariance-event analysis beyond the present Gaussian setting, extend the model to heterogeneous views with different ambient dimensions, and analyze mini-batch or multi-pass SGD. These extensions would move the theory closer to practical contrastive pretraining while preserving the compute-allocation perspective emphasized by efficient reasoning systems.

\newpage
\bibliography{colm2026_conference}
\bibliographystyle{colm2026_conference}

\clearpage
\appendix
\begin{center}
    {\Large\bf Supplementary Material}
\end{center}

\section*{Catalogue of the Supplementary Material}

The supplementary material is organized as follows. The first part collects
notation and proves the main decomposition and scaling theorem from the main
paper. The next parts prove the approximation, GD-bias, and GD-variance bounds
separately. The final parts collect auxiliary probabilistic and spectral lemmas
and provide additional experimental details.

\vspace{0.5em}
\begin{center}
\begin{tabular}{@{}p{0.23\textwidth}p{0.62\textwidth}r@{}}
\toprule
\textbf{Part} & \textbf{Contents} & \textbf{Page} \\
\midrule
Preliminary & Notation, minimizer identities, irreducible risk, and the proofs of the main risk decomposition, main scaling theorem, and compute-allocation heuristic. & \pageref{app:preliminary} \\
Approximation Error & Definition of approximation error and deterministic spectral-tail approximation scaling under the spectral truncation sketch. & \pageref{app:approximation} \\
GD Decomposition & Bias--variance--cross decomposition for empirical GD. & \pageref{sec:contrastive-gd-excess} \\
Bias of Normal GD & General upper and lower bounds for the GD-bias term and their specialization under the source condition. & \pageref{app:gd-bias} \\
Variance of Normal GD & General upper and lower bounds for the GD-variance term and their specialization to the product effective dimension. & \pageref{app:gd-variance} \\
Auxiliary Lemmas & Concentration, block identities, norm equivalence, scalar filter estimates, noise covariance bounds, and effective-dimension estimates used in the proofs. & \pageref{app:auxiliary} \\
Experimental Details & Synthetic-data generation, implementation details, grids, and numerical summaries supporting the experiments. & \pageref{app:experiment-details} \\
Notation Map & Mapping from symbols used in the main text and appendix to their formulas and meanings. & \pageref{app:notation-map} \\
\bottomrule
\end{tabular}
\end{center}
\section{Preliminary}
\label{app:preliminary}

\subsection{Appendix notation and minimizer identities}
\label{subsec:minimizer-justification}
We collect here the notation used throughout the appendix. The full covariance
and cross-covariance are
\[
    H:=\mathbb E[xx^\top]=\mathbb E[yy^\top],
    \qquad
    C:=\mathbb E[xy^\top].
\]
Under Assumption~\mainref{ass:aligned-power-laws}, these operators are
simultaneously diagonalizable. In the common population spectral basis,
\begin{align*}
    H
    &=
    \sum_i
    (\lambda_{z,i}+\lambda_{\epsilon,i})u_iu_i^\top,
    \\
    C
    &=
    \sum_i
    \lambda_{z,i}u_iu_i^\top .
\end{align*}
The full quadratic contrastive risk is
\[
    R(W)
    :=
    -\langle W,C\rangle
    +
    \frac12\operatorname{tr}(W^\top HWH),
\]
with the minimizer
\[
    W^\star:=H^{-1}CH^{-1}.
\]
For the fixed sketch matrix \(S\), write
\begin{align*}
    & \widetilde x=Sx,
    \qquad
    \widetilde y=Sy,
    \qquad
    \Sigma:=SHS^\top,
    \\
    & C_M:=SCS^\top=\mathbb E[\widetilde x\widetilde y^\top].
\end{align*}
The sketched risk and its minimizer are
\[
    R_M(A)
    :=
    -\langle A,C_M\rangle
    +
    \frac12\operatorname{tr}(A^\top\Sigma A\Sigma),
\]
with the minimizer
\[
    A^\star:=\Sigma^{-1}C_M\Sigma^{-1}.
\]
The minimizer identities are justified below. For matrices of compatible dimensions
\[
    \langle B_1,B_2\rangle_{\Sigma,\Sigma}
    :=
    \operatorname{tr}(B_1^\top\Sigma B_2\Sigma),
    \quad
    \|B\|_{\Sigma,\Sigma}^2
    :=
    \langle B,B\rangle_{\Sigma,\Sigma}.
\]
For empirical GD, let empirical marginal sketched covariance
\[
    \widehat\Sigma_x
    :=
    \frac1N\sum_{n=1}^N\widetilde x_n\widetilde x_n^\top,
    \qquad
    \widehat\Sigma_y
    :=
    \frac1N\sum_{n=1}^N\widetilde y_n\widetilde y_n^\top,
\]
and empirical cross sketched covariance
\[
    \widehat C
    :=
    \frac1N\sum_{n=1}^N\widetilde x_n\widetilde y_n^\top.
\]
We use the empirical Hessian, empirical residual, and GD filters
\[
    \widehat{\mathscr H}(B):=\widehat\Sigma_xB\widehat\Sigma_y,
    \qquad
    \widehat E:=\widehat C-\widehat\Sigma_xA^\star\widehat\Sigma_y,
\]
\[
    \widehat{\mathscr B}_{r:s}
    :=
    \prod_{t=r}^s(I-\gamma_t\widehat{\mathscr H}),
    \qquad
    \widehat{\mathscr B}_{L+1:L}:=I,
\]
and
\[
    \widehat{\mathscr B}_L:=\widehat{\mathscr B}_{1:L},
    \qquad
    \widehat{\mathscr V}_L
    :=
    \sum_{t=1}^L\gamma_t\widehat{\mathscr B}_{t+1:L}.
\]
We also use the optimization shorthands
\[
    R:=L_{\mathrm{eff}}\gamma,
    \qquad
    \rho:=R^{-1/2}.
\]
For a sufficiently small absolute constant \(c_{\mathrm{rel}}>0\), define the
covariance event
\[
\mathcal E_{\mathrm{cov}}(R)
:=
\left\{
\max_{\sharp\in\{x,y\}}
\left\|
\Sigma^{-1/2}\widehat\Sigma_{\sharp}\Sigma^{-1/2}-I
\right\|
\le c_{\mathrm{rel}}R^{-1/2}
\right\}.
\]
Under the aligned power-law assumptions, write \(\delta:=a-b\). By
Assumption~\ref{ass:contrastive-source}, for the spectral decomposition
\(\Sigma=\sum_{i=1}^M\mu_i v_iv_i^\top\),
\[
    \mu_i\asymp i^{-b},
    \qquad
    C_M=\sum_{i=1}^M\kappa_i\mu_i v_iv_i^\top,
    \qquad
    \kappa_i\lesssim i^{-\delta}.
\]
For approximation arguments we also use the full and sketched whitened signals
\[
    K:=H^{-1/2}CH^{-1/2},
    \qquad
    K_M:=\Sigma^{-1/2}C_M\Sigma^{-1/2}.
\]

We now justify the minimizer identities used above. For the population risk,
\[
\begin{aligned}
    R(W)
    &=
    -\langle W,C\rangle
    +
    \frac12\operatorname{tr}(W^\top HWH).
\end{aligned}
\]
For any perturbation \(\Delta\in\mathbb R^{D\times D}\), the first variation is
\[
\begin{aligned}
    \mathrm d R(W)[\Delta]
    &=
    -\langle \Delta,C\rangle
    +
    \frac12
    \operatorname{tr}(\Delta^\top HWH)
    +
    \frac12
    \operatorname{tr}(W^\top H\Delta H)  \\
    &=
    -\langle \Delta,C\rangle
    +
    \frac12
    \langle \Delta,HWH\rangle
    +
    \frac12
    \langle \Delta,HWH\rangle  \\
    &=
    \langle \Delta,HWH-C\rangle .
\end{aligned}
\]
Hence
\[
    \nabla R(W)=HWH-C.
\]
The population minimizer therefore satisfies the normal equation
\[
    HW^\star H=C.
\]
Assuming \(H\) is invertible, this gives
\[
    W^\star=H^{-1}CH^{-1}.
\]
Moreover, since the full-population Hessian operator
\[
    \mathscr H_{\mathrm{full}}(W)=HWH
\]
is positive semidefinite with respect to the Frobenius inner product, this stationary point is the global minimizer. 

The proof works similarly to sketched minimizer. From its definition,
\[
\begin{aligned}
    R_M(A)
    &=
    -\langle A,C_M\rangle
    +
    \frac12\operatorname{tr}(A^\top \Sigma A\Sigma).
\end{aligned}
\]
For any perturbation \(\Delta\in\mathbb R^{M\times M}\), we have
\[
\begin{aligned}
    \mathrm d R_M(A)[\Delta]
    &=
    -\langle \Delta,C_M\rangle
    +
    \frac12
    \operatorname{tr}(\Delta^\top \Sigma A\Sigma)
    +
    \frac12
    \operatorname{tr}(A^\top \Sigma \Delta \Sigma)  \\
    &=
    -\langle \Delta,C_M\rangle
    +
    \frac12
    \langle \Delta,\Sigma A\Sigma\rangle
    +
    \frac12
    \langle \Delta,\Sigma A\Sigma\rangle  \\
    &=
    \langle \Delta,\Sigma A\Sigma-C_M\rangle .
\end{aligned}
\]
Therefore,
\[
    \nabla R_M(A)=\Sigma A\Sigma-C_M.
\]
The sketched population minimizer satisfies
\[
    \Sigma A^\star\Sigma=C_M.
\]
Assuming \(\Sigma\) is invertible, we obtain
\[
    A^\star
    =
    \Sigma^{-1}C_M\Sigma^{-1}.
\]
Since
\[
    \Sigma=SHS^\top,
    \qquad
    C_M=SCS^\top,
\]
this can be written as
\[
    A^\star
    =
    (SHS^\top)^{-1}SCS^\top(SHS^\top)^{-1}.
\]
Again, the sketched Hessian operator
\[
    \mathscr H(A)=\Sigma A\Sigma
\]
is positive semidefinite with respect to the Frobenius inner product, so this stationary point is the global minimizer.

\subsection{Irreducible risk}
\label{subsec:irreducible-risk}

We first identify the risk level achieved by the full population minimizer. Since
\[
    W^\star=H^{-1}CH^{-1},
\]
the population normal equation gives
\[
    HW^\star H=C.
\]
Therefore,
\[
\begin{aligned}
    R(W^\star)
    &=
    -\langle W^\star,C\rangle
    +
    \frac12\operatorname{tr}\bigl((W^\star)^\top H W^\star H\bigr)  \\
    &=
    -\langle W^\star,C\rangle
    +
    \frac12
    \langle W^\star,HW^\star H\rangle  \\
    &=
    -\langle W^\star,C\rangle
    +
    \frac12
    \langle W^\star,C\rangle  \\
    &=
    -\frac12
    \langle W^\star,C\rangle .
\end{aligned}
\]
Substituting \(W^\star=H^{-1}CH^{-1}\), we obtain
\[
\begin{aligned}
    R(W^\star)
    &=
    -\frac12
    \operatorname{tr}
    \left[
        (H^{-1}CH^{-1})^\top C
    \right]  \\
    &=
    -\frac12
    \operatorname{tr}
    \left(
        H^{-1}C^\top H^{-1}C
    \right)  \\
    &=
    -\frac12
    \left\|
        H^{-1/2}CH^{-1/2}
    \right\|_F^2 .
\end{aligned}
\]
In general, we call this quantity the irreducible risk and denote
\[
    R_{\mathrm{irr}}
    :=
    R(W^\star)
    =
    -\frac12
    \left\|
        H^{-1/2}CH^{-1/2}
    \right\|_F^2 .
\]
Note that all additional error terms, including approximation error from sketching and excess error from finite-data optimization, are measured
relative to this irreducible baseline.

Additionally, we particularly bound the irreducible error as below. By Assumption~\mainref{ass:aligned-power-laws}, the full marginal covariance and cross-covariance are simultaneously diagonalizable:
\[
    H=\Lambda_z+\Lambda_\epsilon,
    \qquad
    C=\Lambda_z.
\]
Writing
\[
    \Lambda_z=
    \sum_{i\ge1}\lambda_{z,i}u_iu_i^\top,
    \qquad
    \Lambda_\epsilon=
    \sum_{i\ge1}\lambda_{\epsilon,i}u_iu_i^\top,
\]
we have
\[
    H=
    \sum_{i\ge1}
    (\lambda_{z,i}+\lambda_{\epsilon,i})u_iu_i^\top .
\]
Therefore,
\[
    H^{-1/2}CH^{-1/2}
    =
    \sum_{i\ge1}
    \frac{\lambda_{z,i}}{\lambda_{z,i}+\lambda_{\epsilon,i}}
    u_iu_i^\top .
\]
Hence the irreducible risk simplifies to
\[
\begin{aligned}
    R(W^\star)
    &=
    -\frac12
    \left\|
        H^{-1/2}CH^{-1/2}
    \right\|_F^2  \\
    &=
    -\frac12
    \sum_{i\ge 1}
    \left(
        \frac{\lambda_{z,i}}
        {\lambda_{z,i}+\lambda_{\epsilon,i}}
    \right)^2 .
\end{aligned}
\]
Equivalently, defining the population signal-to-marginal ratio
\[
    \kappa_i
    :=
    \frac{\lambda_{z,i}}
    {\lambda_{z,i}+\lambda_{\epsilon,i}},
\]
we obtain
\[
    R(W^\star)
    =
    -\frac12
    \sum_{i\ge 1}\kappa_i^2.
\]
Under Assumption~\ref{ass:aligned-power-laws}, if
\[
    \lambda_{z,i}\asymp i^{-a},
    \qquad
    \lambda_{\epsilon,i}\asymp i^{-b},
    \qquad
    a>b,
\]
then
\[
    \kappa_i
    =
    \frac{\lambda_{z,i}}
    {\lambda_{z,i}+\lambda_{\epsilon,i}}
    \asymp
    i^{-(a-b)}.
\]
Consequently,
\[
    R(W^\star)
    =
    -\frac12
    \sum_{i\ge 1}\kappa_i^2
    \asymp
    -
    \sum_{i\ge 1}i^{-2(a-b)}.
\]
In particular, the irreducible risk is finite whenever
\[
    2(a-b)>1,
    \qquad\text{equivalently}\qquad
    a-b>\frac12.
\]

\subsection{Proof of Proposition~\ref{prop:main-risk-decomposition}}
\label{subsec:proof-main-risk-decomposition}

We first split the risk into its full-population baseline, its sketched approximation gap, and its sketched excess risk:
\[
\begin{aligned}
    R_M(A_L)
    &=
    R(W^\star)
    +
    \left[
        R_M(A^\star)-R(W^\star)
    \right] \\
    &+
    \left[
        R_M(A_L)-R_M(A^\star)
    \right]  \\
    &=
    R_{\mathrm{irr}}
    +
    \operatorname{Approx}
    +
    \left[
        R_M(A_L)-R_M(A^\star)
    \right].
\end{aligned}
\]
Since \(A^\star\) minimizes the sketched population risk, we have
\[
    R_M(A_L)-R_M(A^\star)
    =
    \frac12
    \|A_L-A^\star\|_{\Sigma,\Sigma}^2.
\]
By Lemma~\ref{lem:gd-bias-var-cross},
\[
    R_M(A_L)-R_M(A^\star)
    =
    \operatorname{Bias}^{\mathrm{samp}}_L
    +
    \operatorname{Var}^{\mathrm{samp}}_L
    +
    \operatorname{Cross}^{\mathrm{samp}}_L,
\]
where
\[
    \operatorname{Bias}^{\mathrm{samp}}_L
    :=
    \frac12
    \left\|
        \widehat{\mathscr B}_L(A^\star)
    \right\|_{\Sigma,\Sigma}^2,
\]
\[
    \operatorname{Var}^{\mathrm{samp}}_L
    :=
    \frac12
    \left\|
        \widehat{\mathscr V}_L(\widehat E)
    \right\|_{\Sigma,\Sigma}^2,
\]
and
\[
    \operatorname{Cross}^{\mathrm{samp}}_L
    :=
    -
    \left\langle
        \widehat{\mathscr B}_L(A^\star),
        \widehat{\mathscr V}_L(\widehat E)
    \right\rangle_{\Sigma,\Sigma}.
\]
Taking expectation over the empirical sample gives
\[
    \mathbb E_{\mathcal D}[R_M(A_L)]
    =
    R_{\mathrm{irr}}
    +
    \operatorname{Approx}
    +
    \operatorname{Bias}_L
    +
    \operatorname{Var}_L
    +
    \operatorname{Cross}_L,
\]
with the definitions in Proposition~\ref{prop:main-risk-decomposition}.

It remains to bound the cross term. By Cauchy--Schwarz,
\[
\begin{aligned}
    & |\operatorname{Cross}_L|
    \le
    \mathbb E_{\mathcal D}
    \left[
        \left\|
            \widehat{\mathscr B}_L(A^\star)
        \right\|_{\Sigma,\Sigma}
        \left\|
            \widehat{\mathscr V}_L(\widehat E)
        \right\|_{\Sigma,\Sigma}
    \right]  \\
    &\le
    \left(
        \mathbb E_{\mathcal D}
        \left[
            \left\|
                \widehat{\mathscr B}_L(A^\star)
            \right\|_{\Sigma,\Sigma}^2
        \right]
    \right)^{1/2}
    \left(
        \mathbb E_{\mathcal D}
        \left[
            \left\|
                \widehat{\mathscr V}_L(\widehat E)
            \right\|_{\Sigma,\Sigma}^2
        \right]
    \right)^{1/2}  \\
    &=
    2\sqrt{\operatorname{Bias}_L\operatorname{Var}_L}.
\end{aligned}
\]
Using \(2\sqrt{ab}\le a+b\), we obtain
\[
    |\operatorname{Cross}_L|
    \le
    \operatorname{Bias}_L+\operatorname{Var}_L.
\]
Therefore,
\[
    \mathbb E_{\mathcal D}[R_M(A_L)]
    \le
    R_{\mathrm{irr}}
    +
    \operatorname{Approx}
    +
    2\operatorname{Bias}_L
    +
    2\operatorname{Var}_L,
\]
which proves the proposition.

\subsection{Proof of Theorem~\ref{thm:main-gd-scaling}}
\label{subsec:proof-main-gd-scaling}

By Proposition~\ref{prop:main-risk-decomposition},
\[
    \mathbb E_{\mathcal D}[R_M(A_L)]
    =
    R_{\mathrm{irr}}
    +
    \operatorname{Approx}
    +
    \operatorname{Bias}_L
    +
    \operatorname{Var}_L
    +
    \operatorname{Cross}_L.
\]
We substitute the specific estimates for the four non-baseline terms. Since
\(1\lesssim R\lesssim N/M\), Lemma~\ref{lem:covariance-event-concentration}
implies that, conditional on the fixed sketch, \(\mathcal E_{\mathrm{cov}}(R)\)
holds with probability at least \(1-\exp(-\Omega(M))\), and also implies
\(N\gtrsim M\), as required by the variance upper bound. The condition
\(R\lesssim M^{2b}\) is equivalent to the bias-unsaturated regime
\(R^{1/(2b)}\lesssim M\). We work on this covariance event.

First, since \(\frac12<\delta<b+\frac12\), Theorem~\ref{thm:approx-power-law}
gives
\[
    \operatorname{Approx}
    =
    \Theta\left(M^{1-2\delta}\right).
\]
Second, under the bias-unsaturated condition \(R^{1/(2b)}\lesssim M\),
Theorem~\ref{thm:gd-bias-specific} gives
\[
    \operatorname{Bias}_L
    =
    \Theta\left(R^{\frac{1-2\delta}{2b}}\right),
    \qquad
    R=L_{\mathrm{eff}}\gamma.
\]
Define
\[
    D_{\times}(R,M)
    :=
    \begin{cases}
    R^{1/b}\log(eR), & 1\le R^{1/b}\le M, \\
    R^{1/b}\left(1+\log\dfrac{M^2}{R^{1/b}}\right),
        & M<R^{1/b}<M^2, \\
    M^2, & R^{1/b}\ge M^2.
    \end{cases}
\]
Third, since the theorem assumes \(R\gtrsim1\),
Theorem~\ref{thm:gd-variance-specific}, based on the general lower bound in
Lemma~\ref{lem:gd-variance-lower-general}, gives
\[
    \operatorname{Var}_L
    =
    \Theta\left(
    \frac{D_{\times}(R,M)}{N}
    \right).
\]
Finally, by Cauchy--Schwarz and the preceding bias and variance estimates,
\[
\begin{aligned}
    \operatorname{Cross}_L
    &=
    \mathcal O\left(
        \sqrt{\operatorname{Bias}_L\operatorname{Var}_L}
    \right) \\
    &=
    \mathcal O\left(
    \sqrt{
        R^{\frac{1-2\delta}{2b}}
        \cdot
        \frac{D_{\times}(R,M)}{N}
    }
    \right).
\end{aligned}
\]
Combining these displays proves
\[
\begin{aligned}
    \mathbb E_{\mathcal D}
    \left[
        R_M(A_L)
    \right]
    &=
    R_{\mathrm{irr}}
    +
    \Theta\left(M^{1-2\delta}\right)
    +
    \Theta\left(R^{\frac{1-2\delta}{2b}}\right) \\
    &\quad
    +
    \Theta\left(
    \frac{D_{\times}(R,M)}{N}
    \right) \\
    &\quad
    +
    \mathcal O\left(
    \sqrt{
        R^{\frac{1-2\delta}{2b}}
        \cdot
        \frac{D_{\times}(R,M)}{N}
    }
    \right),
\end{aligned}
\]
which is the claimed scaling law.

\subsection{Proof of Proposition~\ref{prop:compute-allocation}}
\label{subsec:proof-compute-allocation}

This result is a compute-allocation heuristic rather than a minimax optimality
theorem. Throughout the proof, \(\approx\), \(\asymp\), and the displayed power
laws suppress logarithmic factors from
\(L_{\mathrm{eff}}=\lfloor L/\log L\rfloor\), from the boundary logarithms in
\(D_{\times}(R,M)\), and from implementation-dependent constants in the compute
proxy. Since \(\gamma\asymp1\), we identify
\[
    \mathcal C
    \asymp
    LNM^2
    \approx
    RNM^2.
\]
Let
\[
    q:=2\delta-1.
\]
Theorem~\ref{thm:main-gd-scaling} gives the leading excess-risk proxy
\[
    \mathcal E(M,N,R)
    \asymp
    M^{-q}
    +
    R^{-q/(2b)}
    +
    \frac{D_{\times}(R,M)}{N}.
\]
The covariance-concentration condition in Theorem~\ref{thm:main-gd-scaling}
requires
\[
    R\lesssim \frac{N}{M},
    \qquad\text{or equivalently}\qquad
    N\gtrsim RM.
\]
The remaining horizon constraint is the bias-unsaturation condition
\[
    R\lesssim M^{2b},
    \qquad\text{equivalently}\qquad
    M\gtrsim R^{1/(2b)}.
\]

We first record the product-unsaturated boundary allocation. This is not the
unconstrained optimum of the proxy; it records the cost of additionally forcing
the product effective dimension to remain in its first, unsaturated branch. If
one insists on
\[
    R^{1/b}\lesssim M,
\]
the smallest admissible model size is \(M\asymp R^{1/b}\). Taking the covariance
constraint at equality, \(N\asymp RM\), gives
\[
    N\asymp R^{1+1/b},
    \qquad
    \mathcal C
    \approx
    R^2R^{3/b}
    =
    R^{(2b+3)/b}.
\]
Hence
\[
    R\asymp \mathcal C^{\frac{b}{2b+3}},
    \qquad
    M\asymp \mathcal C^{\frac{1}{2b+3}},
    \qquad
    N\asymp \mathcal C^{\frac{b+1}{2b+3}}.
\]
On this boundary, \(D_{\times}(R,M)\asymp R^{1/b}\), so the variance term is
\(R^{-1}\), which is lower order than the bias term \(R^{-q/(2b)}\) for
\(0<q<2b\).

For the balanced allocations, the approximation and bias terms impose
\[
    M^{-q}
    \asymp
    R^{-q/(2b)},
    \qquad
    M\asymp R^{1/(2b)}.
\]
This is exactly the smallest model size allowed by the bias-unsaturation
constraint. At this boundary, \(D_{\times}(R,M)\asymp R^{1/b}\) up to logarithmic
factors.

It remains to choose \(N\). The covariance constraint imposes the lower bound
\(N\gtrsim RM\asymp R^{1+1/(2b)}\). On the other hand, balancing the variance
term with the common approximation--bias scale \(R^{-q/(2b)}\) would require
\[
    N\asymp R^{q/(2b)+1/b}.
\]
Thus the leading sample allocation is
\[
    N\asymp
    \max\left\{R^{1+1/(2b)},\,R^{q/(2b)+1/b}\right\},
\]
and the comparison of these two powers gives the threshold
\(q=2b-1\).

If \(0<q<2b-1\), then taking the covariance constraint at equality,
\(N\asymp RM\), makes the variance term
\[
    \frac{D_{\times}(R,M)}{N}
    \asymp
    \frac{R^{1/b}}{R^{1+1/(2b)}}
    =
    R^{-1+1/(2b)},
\]
which is lower order than \(R^{-q/(2b)}\). The compute proxy becomes
\[
    \mathcal C
    \approx
    R^2M^3
    \asymp
    R^{(4b+3)/(2b)}.
\]
Solving for \(R,M,N\) gives
\[
    R\asymp \mathcal C^{\frac{2b}{4b+3}},
    \qquad
    M\asymp \mathcal C^{\frac{1}{4b+3}},
    \qquad
    N\asymp \mathcal C^{\frac{2b+1}{4b+3}}.
\]

If \(2b-1<q<2b\), the same covariance-saturated allocation would leave the
variance term leading. We therefore balance all three leading terms while keeping
\(M\asymp R^{1/(2b)}\):
\[
    R^{-q/(2b)}
    \asymp
    \frac{R^{1/b}}{N}.
\]
Thus
\[
    N\asymp R^{q/(2b)+1/b}.
\]
The compute proxy gives
\[
    \mathcal C
    \approx
    RNM^2
    \asymp
    R\cdot R^{q/(2b)+1/b}\cdot R^{1/b}
    =
    R^{(2b+4+q)/(2b)}.
\]
Consequently,
\[
    R\asymp \mathcal C^{\frac{2b}{2b+4+q}},
    \qquad
    M\asymp \mathcal C^{\frac{1}{2b+4+q}},
    \qquad
    N\asymp \mathcal C^{\frac{q+2}{2b+4+q}}.
\]
The condition \(q>2b-1\) is exactly what ensures
\(N\gtrsim RM\) for this variance-balanced allocation, since
\[
    \frac{N}{RM}
    \asymp
    R^{(q-(2b-1))/(2b)}.
\]
At the boundary
\(q=2b-1\), the covariance-limited and variance-balanced formulas coincide.
Finally, \(R=L_{\mathrm{eff}}\gamma\) and
\(L_{\mathrm{eff}}=\lfloor L/\log L\rfloor\), so suppressing floors and
logarithmic factors gives \(L/\log L\approx R/\gamma\). This completes the
claimed allocation rules.

\section{Approximation Error}
\label{app:approximation}

\subsection{Approximation error definition}
The approximation error is the gap between the best sketched model and the
best full model:
\[
\begin{aligned}
    \operatorname{Approx}
    &:=
    R_M(A^\star)-R(W^\star) \\
    &=
    \min_{A\in\mathbb R^{M\times M}} R_M(A)
    -
    \min_{W\in\mathbb R^{D\times D}} R(W).
\end{aligned}
\]
Since \(R_M(A)=R(S^\top A S)\), this is equivalently the gap between the
best risk achievable inside the sketched model class \(S^\top A S\) and the
best unrestricted full risk.
Define the full and sketched whitened signals by
\begin{align*}
    K&:=H^{-1/2}CH^{-1/2}, \\
    K_M&:=\Sigma^{-1/2}C_M\Sigma^{-1/2}.
\end{align*}
Then the approximation error admits the exact representation
\[
    \operatorname{Approx}
    =
    \frac12
    \left(
        \|K\|_F^2-
        \|K_M\|_F^2
    \right).
\]
Indeed, starting from the original definition and using the sketched normal
equation \(\Sigma A^\star\Sigma=C_M\),
\[
\begin{aligned}
    R_M(A^\star)
    &=-\langle A^\star,C_M\rangle
      +\frac12\operatorname{tr}((A^\star)^\top\Sigma A^\star\Sigma) \\
    &=-\langle A^\star,C_M\rangle
      +\frac12\langle A^\star,C_M\rangle \\
    &=-\frac12\langle A^\star,C_M\rangle \\
    &=-\frac12
      \operatorname{tr}\left[(\Sigma^{-1}C_M\Sigma^{-1})^\top C_M\right] \\
    &=-\frac12
      \left\|\Sigma^{-1/2}C_M\Sigma^{-1/2}\right\|_F^2
     =-\frac12\|K_M\|_F^2 .
\end{aligned}
\]
Similarly, the full normal equation \(HW^\star H=C\) gives
\[
\begin{aligned}
    R(W^\star)
    &=-\langle W^\star,C\rangle
      +\frac12\operatorname{tr}((W^\star)^\top H W^\star H) \\
    &=-\frac12\langle W^\star,C\rangle \\
    &=-\frac12
      \operatorname{tr}\left[(H^{-1}CH^{-1})^\top C\right] \\
    &=-\frac12
      \left\|H^{-1/2}CH^{-1/2}\right\|_F^2
     =-\frac12\|K\|_F^2 .
\end{aligned}
\]
Therefore
\[
    R_M(A^\star)-R(W^\star)
    =
    \frac12
    \left(
        \|K\|_F^2-
        \|K_M\|_F^2
    \right),
\]
which proves the claimed representation.

\subsection{Approximation scaling under spectral alignment}

\begin{theorem}[Approximation scaling under the spectral truncation sketch]
\label{thm:approx-power-law}
Suppose Assumptions~\mainref{ass:aligned-power-laws} and
\mainref{ass:contrastive-source} hold. Then
\[
    \operatorname{Approx}
    \asymp
    \sum_{i>M}
    \left(
        \frac{\lambda_{z,i}}
        {\lambda_{z,i}+\lambda_{\epsilon,i}}
    \right)^2
    \asymp
    M^{1-2\delta},
    \qquad
    \delta=a-b.
\]
Equivalently,
\[
    \operatorname{Approx}
    \asymp
    M^{1-2(a-b)}.
\]
\end{theorem}

\begin{proof}
We begin with the original approximation definition proved above:
\[
    \operatorname{Approx}
    =
    R_M(A^\star)-R(W^\star)
    =
    \frac12
    \left(
        \|K\|_F^2-
        \|K_M\|_F^2
    \right).
\]
By Assumption~\mainref{ass:aligned-power-laws}, \(H\) and \(C\) are simultaneously diagonalizable in
the common population spectral basis \(\{u_i\}_{i=1}^D\), with
\[
    H=
    \sum_{i=1}^D
    (\lambda_{z,i}+\lambda_{\epsilon,i})u_iu_i^\top,
    \qquad
    C=
    \sum_{i=1}^D
    \lambda_{z,i}u_iu_i^\top .
\]
Therefore the whitened signal is also diagonal in this basis:
\[
    K
    =
    H^{-1/2}CH^{-1/2}
    =
    \sum_{i=1}^D
    \frac{\lambda_{z,i}}{\lambda_{z,i}+\lambda_{\epsilon,i}}
    u_iu_i^\top .
\]
Since \(a>b\), we have
\(\lambda_{z,i}+\lambda_{\epsilon,i}\asymp i^{-b}\), and hence
\[
    \frac{\lambda_{z,i}}{\lambda_{z,i}+\lambda_{\epsilon,i}}
    \asymp
    \frac{i^{-a}}{i^{-b}}
    =
    i^{-(a-b)}
    =
    i^{-\delta}.
\]

Assumption~\mainref{ass:contrastive-source} states directly that the sketched
covariance pair preserves the leading \(M\) population spectral directions and
their order:
\[
\begin{aligned}
    \Sigma
    &=
    \sum_{i=1}^M \mu_i v_iv_i^\top, \\
    C_M
    &=
    \sum_{i=1}^M \kappa_i\mu_i v_iv_i^\top,
    \qquad
    \kappa_i=
    \frac{\lambda_{z,i}}{\lambda_{z,i}+\lambda_{\epsilon,i}} .
\end{aligned}
\]
The same assumption also places every population eigendirection with order
larger than \(M\) in the null space of \(S\), so these directions do not appear
in the sketched covariance pair.
Therefore the sketched whitened signal is
\[
    K_M
    =
    \Sigma^{-1/2}C_M\Sigma^{-1/2}
    =
    \sum_{i=1}^M
    \frac{\lambda_{z,i}}{\lambda_{z,i}+\lambda_{\epsilon,i}}
    v_iv_i^\top
    =
    \sum_{i=1}^M \kappa_i v_iv_i^\top .
\]
Thus sketching removes exactly the whitened population tail, and
\[
    \operatorname{Approx}
    =
    \frac12
    \left(
        \|K\|_F^2-
        \|K_M\|_F^2
    \right)
    =
    \frac12
    \sum_{i>M}
    \left(
        \frac{\lambda_{z,i}}
        {\lambda_{z,i}+\lambda_{\epsilon,i}}
    \right)^2 .
\]
Finally, since \(\delta>1/2\),
\[
    \sum_{i>M}
    \left(
        \frac{\lambda_{z,i}}
        {\lambda_{z,i}+\lambda_{\epsilon,i}}
    \right)^2
    \asymp
    \sum_{i>M}i^{-2\delta}
    \asymp
    M^{1-2\delta}.
\]
This completes the proof.
\end{proof}
% ============================================================
% Excess-risk decomposition for empirical GD
% ============================================================

\section{Excess risk of empirical gradient descent}
\label{sec:contrastive-gd-excess}

We use the sketched risk \(R_M\), its minimizer \(A^\star\), the empirical
residual \(\widehat E:=\widehat C-\widehat\Sigma_xA^\star\widehat\Sigma_y\),
and the empirical GD filters \(\widehat{\mathscr B}_L\) and
\(\widehat{\mathscr V}_L\). The relevant quantities for this section are the
following bias, variance, and cross terms.

\begin{lemma}[Bias--variance--cross decomposition for empirical GD]
\label{lem:gd-bias-var-cross}
Let \(A_L\) be the output of empirical GD after \(L\) steps. Then
\[
    R_M(A_L)-R_M(A^\star)
    =
    \operatorname{Bias}^{\mathrm{samp}}_L
    +
    \operatorname{Var}^{\mathrm{samp}}_L
    +
    \operatorname{Cross}^{\mathrm{samp}}_L,
\]
where
\[
    \operatorname{Bias}^{\mathrm{samp}}_L
    :=
    \frac12
    \left\|
        \widehat{\mathscr B}_{L}(A^\star)
    \right\|_{\Sigma,\Sigma}^2,
\]
\[
    \operatorname{Var}^{\mathrm{samp}}_L
    :=
    \frac12
    \left\|
        \widehat{\mathscr V}_{L}(\widehat E)
    \right\|_{\Sigma,\Sigma}^2,
\]
and
\[
    \operatorname{Cross}^{\mathrm{samp}}_L
    :=
    -
    \left\langle
        \widehat{\mathscr B}_{L}(A^\star),
        \widehat{\mathscr V}_{L}(\widehat E)
    \right\rangle_{\Sigma,\Sigma}.
\]
\end{lemma}

\begin{proof}
Let
\[
    \Delta_t:=A_t-A^\star.
\]
Using the GD recursion and substituting \(A_{t-1}=\Delta_{t-1}+A^\star\), we obtain
\[
    \Delta_t
    =
    \Delta_{t-1}
    -
    \gamma_t
    \widehat\Sigma_x\Delta_{t-1}\widehat\Sigma_y
    +
    \gamma_t
    \left(
        \widehat C-\widehat\Sigma_xA^\star\widehat\Sigma_y
    \right).
\]
By the definition of \(\widehat{\mathscr H}\) and \(\widehat E\), this becomes
\[
    \Delta_t
    =
    (I-\gamma_t\widehat{\mathscr H})(\Delta_{t-1})
    +
    \gamma_t\widehat E.
\]
Iterating the recursion from \(t=1\) to \(L\), and using
\[
    \Delta_0=A_0-A^\star=-A^\star,
\]
gives
\[
    \Delta_L
    =
    -\widehat{\mathscr B}_L(A^\star)
    +
    \widehat{\mathscr V}_L(\widehat E).
\]
Therefore,
\[
    R_M(A_L)-R_M(A^\star)
    =
    \frac12\|\Delta_L\|_{\Sigma,\Sigma}^2.
\]
Substituting the previous display and expanding the square gives
\[
    \frac12
    \left\|
        -\widehat{\mathscr B}_L(A^\star)
        +
        \widehat{\mathscr V}_L(\widehat E)
    \right\|_{\Sigma,\Sigma}^2
\]
\[
    =
    \frac12
    \left\|
        \widehat{\mathscr B}_L(A^\star)
    \right\|_{\Sigma,\Sigma}^2
    +
    \frac12
    \left\|
        \widehat{\mathscr V}_L(\widehat E)
    \right\|_{\Sigma,\Sigma}^2
    -
    \left\langle
        \widehat{\mathscr B}_L(A^\star),
        \widehat{\mathscr V}_L(\widehat E)
    \right\rangle_{\Sigma,\Sigma}.
\]
This proves the desired decomposition.
\end{proof}

\begin{remark}[The cross term is harmless for upper bounds]
\label{rem:cross-term-upper}
From the GD recursion, we have
\[
    A_L-A^\star
    =
    -
    \widehat{\mathscr B}_L(A^\star)
    +
    \widehat{\mathscr V}_L(\widehat E).
\]
Therefore, expanding the squared population norm gives
\[
\begin{aligned}
    \frac12
    \|A_L-A^\star\|_{\Sigma,\Sigma}^2
    &=
    \frac12
    \left\|
        \widehat{\mathscr B}_L(A^\star)
    \right\|_{\Sigma,\Sigma}^2
    +
    \frac12
    \left\|
        \widehat{\mathscr V}_L(\widehat E)
    \right\|_{\Sigma,\Sigma}^2  \\
    &\quad
    -
    \left\langle
        \widehat{\mathscr B}_L(A^\star),
        \widehat{\mathscr V}_L(\widehat E)
    \right\rangle_{\Sigma,\Sigma}.
\end{aligned}
\]
The last term is the cross term. Although it does not necessarily vanish, it is harmless for upper bounds. Indeed, by Cauchy--Schwarz,
\[
\begin{aligned}
    \left|
    \left\langle
        \widehat{\mathscr B}_L(A^\star),
        \widehat{\mathscr V}_L(\widehat E)
    \right\rangle_{\Sigma,\Sigma}
    \right|
    &\le
    \left\|
        \widehat{\mathscr B}_L(A^\star)
    \right\|_{\Sigma,\Sigma}
    \left\|
        \widehat{\mathscr V}_L(\widehat E)
    \right\|_{\Sigma,\Sigma}  \\
    &\le
    \frac12
    \left\|
        \widehat{\mathscr B}_L(A^\star)
    \right\|_{\Sigma,\Sigma}^2
    +
    \frac12
    \left\|
        \widehat{\mathscr V}_L(\widehat E)
    \right\|_{\Sigma,\Sigma}^2 .
\end{aligned}
\]
Consequently,
\[
    \frac12
    \|A_L-A^\star\|_{\Sigma,\Sigma}^2
    \le
    \left\|
        \widehat{\mathscr B}_L(A^\star)
    \right\|_{\Sigma,\Sigma}^2
    +
    \left\|
        \widehat{\mathscr V}_L(\widehat E)
    \right\|_{\Sigma,\Sigma}^2 .
\]
Taking expectation, we obtain
\[
    \mathbb E
    \left[
        \frac12
        \|A_L-A^\star\|_{\Sigma,\Sigma}^2
    \right]
    \lesssim
    \operatorname{Bias}_L+\operatorname{Var}_L,
\]
where \(\operatorname{Bias}_L\) and \(\operatorname{Var}_L\) are the expected
quantities in Proposition~\mainref{prop:main-risk-decomposition}. Therefore, for the
purpose of proving an upper bound on the excess risk, it is sufficient to
control the bias and variance terms separately. The cross term can be absorbed
into their sum and does not affect the upper-bound scaling.
\end{remark}
\section{Bias of Normal GD}
\label{app:gd-bias}

% ============================================================
% General upper bound for the GD bias
% ============================================================

\subsection{A general upper bound for the GD bias}
\label{subsec:gd-bias-upper}

The GD-bias term is
\[
    \operatorname{Bias}_L
    :=
    \frac12
    \left\|
        \widehat{\mathscr B}_L(A^\star)
    \right\|_{\Sigma,\Sigma}^2.
\]

For \(k\in\{1,\ldots,M\}\), define the population spectral projector
\[
    P_k:=\sum_{i\le k}v_iv_i^\top.
\]

\begin{lemma}[General GD-bias upper bound]
\label{lem:gd-bias-upper-general}
Suppose Assumptions~\mainref{ass:gd-schedule} and
\mainref{ass:contrastive-source} hold, and suppose that
\(\mathcal E_{\mathrm{cov}}(R)\) occurs and \(\|\Sigma\|\lesssim1\). Then, for
every \(1\le k\le M\),
\[
    \operatorname{Bias}_L
    \lesssim
    \frac{1}{L_{\mathrm{eff}}\gamma}
    \sum_{i\le k}\frac{\kappa_i^2}{\mu_i^2}
    +
    \sum_{i>k}\kappa_i^2 .
\]
Before substituting the source condition,
\[
    \operatorname{Bias}_L
    \lesssim
    \frac{1}{L_{\mathrm{eff}}\gamma}
    \|P_kA^\star P_k\|_F^2
    +
    \|A^\star-P_kA^\star P_k\|_{\Sigma,\Sigma}^2.
\]
\end{lemma}

\begin{proof}
Define the population head and tail parts of \(A^\star\) by
\[
    A^\star_{\mathrm h}:=P_kA^\star P_k,
    \qquad
    A^\star_{\mathrm t}:=A^\star-P_kA^\star P_k.
\]
By linearity of \(\widehat{\mathscr B}_L\),
\[
    \widehat{\mathscr B}_L(A^\star)
    =
    \widehat{\mathscr B}_L(A^\star_{\mathrm h})
    +
    \widehat{\mathscr B}_L(A^\star_{\mathrm t}).
\]
Using \(\|B_1+B_2\|_{\Sigma,\Sigma}^2
\le 2\|B_1\|_{\Sigma,\Sigma}^2
+2\|B_2\|_{\Sigma,\Sigma}^2\), we get
\[
    \operatorname{Bias}_L
    \lesssim
    \left\|
        \widehat{\mathscr B}_L(A^\star_{\mathrm h})
    \right\|_{\Sigma,\Sigma}^2
    +
    \left\|
        \widehat{\mathscr B}_L(A^\star_{\mathrm t})
    \right\|_{\Sigma,\Sigma}^2.
\]

We first bound the head term. By Lemma~\ref{lem:covariance-event-norm-equivalence},
\[
    \left\|
        \widehat{\mathscr B}_L(A^\star_{\mathrm h})
    \right\|_{\Sigma,\Sigma}^2
    \lesssim
    \left\|
        \widehat{\mathscr B}_L(A^\star_{\mathrm h})
    \right\|_{\widehat\Sigma_x,\widehat\Sigma_y}^2.
\]
By Lemma~\ref{lem:empirical-product-filter},
\[
    \left\|
        \widehat{\mathscr B}_L(A^\star_{\mathrm h})
    \right\|_{\widehat\Sigma_x,\widehat\Sigma_y}^2
    \lesssim
    \frac{1}{L_{\mathrm{eff}}\gamma}
    \|A^\star_{\mathrm h}\|_F^2.
\]
Therefore,
\[
    \left\|
        \widehat{\mathscr B}_L(A^\star_{\mathrm h})
    \right\|_{\Sigma,\Sigma}^2
    \lesssim
    \frac{1}{L_{\mathrm{eff}}\gamma}
    \|P_kA^\star P_k\|_F^2.
\]

We next bound the tail term. Again by
Lemma~\ref{lem:covariance-event-norm-equivalence},
\[
    \left\|
        \widehat{\mathscr B}_L(A^\star_{\mathrm t})
    \right\|_{\Sigma,\Sigma}^2
    \lesssim
    \left\|
        \widehat{\mathscr B}_L(A^\star_{\mathrm t})
    \right\|_{\widehat\Sigma_x,\widehat\Sigma_y}^2.
\]
By the contraction part of Lemma~\ref{lem:empirical-product-filter},
\[
    \left\|
        \widehat{\mathscr B}_L(A^\star_{\mathrm t})
    \right\|_{\widehat\Sigma_x,\widehat\Sigma_y}^2
    \le
    \|A^\star_{\mathrm t}\|_{\widehat\Sigma_x,\widehat\Sigma_y}^2.
\]
Applying Lemma~\ref{lem:covariance-event-norm-equivalence} in the other
direction gives
\[
    \|A^\star_{\mathrm t}\|_{\widehat\Sigma_x,\widehat\Sigma_y}^2
    \lesssim
    \|A^\star_{\mathrm t}\|_{\Sigma,\Sigma}^2.
\]
Thus
\[
    \left\|
        \widehat{\mathscr B}_L(A^\star_{\mathrm t})
    \right\|_{\Sigma,\Sigma}^2
    \lesssim
    \|A^\star-P_kA^\star P_k\|_{\Sigma,\Sigma}^2.
\]

Combining the head and tail bounds yields
\[
    \operatorname{Bias}_L
    \lesssim
    \frac{1}{L_{\mathrm{eff}}\gamma}
    \|P_kA^\star P_k\|_F^2
    +
    \|A^\star-P_kA^\star P_k\|_{\Sigma,\Sigma}^2.
\]

It remains to use Assumption~\mainref{ass:contrastive-source}. Under this source
condition,
\[
    C_M
    =
    \sum_{i=1}^M
    \kappa_i\mu_i v_iv_i^\top.
\]
Since
\[
    A^\star
    =
    \Sigma^{-1}C_M\Sigma^{-1},
\]
we have
\[
    A^\star
    =
    \sum_{i=1}^M
    \frac{\kappa_i}{\mu_i}v_iv_i^\top.
\]
Therefore,
\[
    P_kA^\star P_k
    =
    \sum_{i\le k}
    \frac{\kappa_i}{\mu_i}v_iv_i^\top,
\]
and hence
\[
    \|P_kA^\star P_k\|_F^2
    =
    \sum_{i\le k}
    \frac{\kappa_i^2}{\mu_i^2}.
\]
Similarly,
\[
    A^\star-P_kA^\star P_k
    =
    \sum_{i>k}
    \frac{\kappa_i}{\mu_i}v_iv_i^\top.
\]
Thus
\[
    \|A^\star-P_kA^\star P_k\|_{\Sigma,\Sigma}^2
    =
    \sum_{i>k}
    \mu_i^2
    \left(
        \frac{\kappa_i}{\mu_i}
    \right)^2
    =
    \sum_{i>k}\kappa_i^2.
\]
Substituting these two identities into the general head--tail bound gives
\[
    \operatorname{Bias}_L
    \lesssim
    \frac{1}{L_{\mathrm{eff}}\gamma}
    \sum_{i\le k}\frac{\kappa_i^2}{\mu_i^2}
    +
    \sum_{i>k}\kappa_i^2.
\]
This completes the proof.
\end{proof}
\subsection{A general lower bound for the GD bias}
\label{subsec:gd-bias-lower}

The GD-bias term is
\[
    \operatorname{Bias}_L
    :=
    \frac12
    \left\|
        \widehat{\mathscr B}_L(A^\star)
    \right\|_{\Sigma,\Sigma}^2 .
\]
For \(k\in\{1,\ldots,M\}\), define
\[
    P_k:=\sum_{i\le k}v_iv_i^\top,
    \qquad
    Q_k:=I-P_k.
\]
We also define the diagonal tail subspace
\[
    \mathcal T_k
    :=
    \operatorname{span}\{v_iv_i^\top:i>k\}.
\]
Let \(\Pi_{\mathcal T_k}\) denote the orthogonal projection onto
\(\mathcal T_k\) under the \(\Sigma,\Sigma\)-inner product.

\begin{lemma}[General GD-bias lower bound]
\label{lem:gd-bias-lower-general}
Suppose Assumptions~\mainref{ass:gd-schedule} and
\mainref{ass:contrastive-source} hold, and suppose that
\(\mathcal E_{\mathrm{cov}}(R)\) occurs and \(\|\Sigma\|\lesssim1\). Let
\[
    \rho:=(L_{\mathrm{eff}}\gamma)^{-1/2}.
\]
Choose \(k\) such that
\[
    \mu_{k+1}\le c_0\rho
\]
for a sufficiently small absolute constant \(c_0>0\). Define
\[
    K_{\mathrm h}:=
    \sum_{i\le k}\kappa_i v_iv_i^\top,
    \qquad
    K_{\mathrm t}:=
    \sum_{i>k}\kappa_i v_iv_i^\top.
\]
Assume that
\[
    c_{\mathrm{rel}}\|K_{\mathrm t}\|_F
    +
    c_{\mathrm{rel}}^2\rho^2\|K_{\mathrm h}\|_F
    \le
    \eta\|K_{\mathrm t}\|_F
\]
for a sufficiently small absolute constant \(\eta>0\). Then
\[
    \operatorname{Bias}_L
    \gtrsim
    \sum_{i>k}\kappa_i^2 .
\]
\end{lemma}

\begin{proof}
We work in the whitened coordinates
\[
    Z_B:=\Sigma^{1/2}B\Sigma^{1/2}.
\]
In these coordinates,
\[
    \|B\|_{\Sigma,\Sigma}
    =
    \|Z_B\|_F.
\]
By Assumption~\mainref{ass:contrastive-source},
\[
    C_M
    =
    \sum_{i=1}^M
    \kappa_i\mu_i v_iv_i^\top,
\]
and therefore
\[
    A^\star
    =
    \Sigma^{-1}C_M\Sigma^{-1}
    =
    \sum_{i=1}^M
    \frac{\kappa_i}{\mu_i}v_iv_i^\top.
\]
Thus
\[
    K_\star
    :=
    \Sigma^{1/2}A^\star\Sigma^{1/2}
    =
    \sum_{i=1}^M
    \kappa_i v_iv_i^\top
    =
    K_{\mathrm h}+K_{\mathrm t},
\]
where
\[
    K_{\mathrm h}:=
    \sum_{i\le k}\kappa_i v_iv_i^\top,
    \qquad
    K_{\mathrm t}:=
    \sum_{i>k}\kappa_i v_iv_i^\top.
\]

We next rewrite the empirical and population Hessian operators in the whitened
coordinates. Define
\[
    \mathscr H(B):=\Sigma B\Sigma,
    \qquad
    \mathscr B_L
    :=
    \prod_{t=1}^L(I-\gamma_t\mathscr H).
\]
In these whitened coordinates, the population Hessian becomes
\[
    \mathscr H(Z):=\Sigma Z\Sigma.
\]
On the event \(\mathcal E_{\mathrm{cov}}(R)\), we can write
\[
    \widehat\Sigma_x
    =
    \Sigma^{1/2}(I+E_x)\Sigma^{1/2},
    \qquad
    \widehat\Sigma_y
    =
    \Sigma^{1/2}(I+E_y)\Sigma^{1/2},
\]
where
\[
    \|E_x\|\vee\|E_y\|
    \le
    c_{\mathrm{rel}}\rho.
\]
Hence the empirical Hessian becomes
\[
    \widehat{\mathscr H}_{\mathrm w}(Z)
    =
    \Sigma(I+E_x)Z(I+E_y)\Sigma.
\]
Equivalently,
\[
    \widehat{\mathscr H}_{\mathrm w}(Z)-\mathscr H(Z)
    =
    \Sigma E_xZ\Sigma
    +
    \Sigma ZE_y\Sigma
    +
    \Sigma E_xZE_y\Sigma.
\]
Let
\[
    \widehat{\mathscr B}^{\mathrm w}_L
    :=
    \prod_{t=1}^L(I-\gamma_t\widehat{\mathscr H}_{\mathrm w})
\]
be the empirical filter in these whitened coordinates. Then
for any matrix \(B\), with \(Z_B=\Sigma^{1/2}B\Sigma^{1/2}\),
\[
\begin{aligned}
    \Sigma^{1/2}
    \left(\widehat\Sigma_xB\widehat\Sigma_y\right)
    \Sigma^{1/2}
    &={}
    \Sigma(I+E_x)Z_B(I+E_y)\Sigma \\
    &={}
    \widehat{\mathscr H}_{\mathrm w}(Z_B).
\end{aligned}
\]
Therefore each one-step empirical GD map is conjugated by
\(B\mapsto\Sigma^{1/2}B\Sigma^{1/2}\):
\[
    \Sigma^{1/2}
    \left(I-\gamma_t\widehat\Sigma_x(\cdot)\widehat\Sigma_y\right)(B)
    \Sigma^{1/2}
    =
    \left(I-\gamma_t\widehat{\mathscr H}_{\mathrm w}\right)(Z_B).
\]
Iterating this identity over \(t=1,\ldots,L\) gives
\[
    \Sigma^{1/2}
    \widehat{\mathscr B}_L(A^\star)
    \Sigma^{1/2}
    =
    \widehat{\mathscr B}^{\mathrm w}_L(K_\star).
\]

We first lower-bound the population evolution of the tail source. Since
\(K_{\mathrm t}\) is diagonal in the eigenbasis of \(\Sigma\), the population
operator acts coordinatewise:
\[
    \mathscr H(v_iv_i^\top)
    =
    \mu_i^2 v_iv_i^\top.
\]
Therefore,
\[
    \mathscr B_L(K_{\mathrm t})
    =
    \sum_{i>k}
    \kappa_i
    \psi_L(\mu_i^2)v_iv_i^\top,
\]
where
\[
    \psi_L(s)
    :=
    \prod_{t=1}^L(1-\gamma_t s).
\]
Since \(\mu_{k+1}\le c_0\rho\), for every \(i>k\),
\[
    \mu_i^2
    \le
    c_0^2\rho^2
    =
    \frac{c_0^2}{L_{\mathrm{eff}}\gamma}.
\]
Taking \(c_0>0\) sufficiently small and using
Assumption~\mainref{ass:gd-schedule}, we have
\[
    \psi_L(\mu_i^2)^2\ge c
\]
for an absolute constant \(c>0\). Hence
\[
    \|\mathscr B_L(K_{\mathrm t})\|_F^2
    =
    \sum_{i>k}
    \kappa_i^2\psi_L(\mu_i^2)^2
    \gtrsim
    \sum_{i>k}\kappa_i^2.
\]

We now compare the empirical evolution with the population evolution after
projecting onto the diagonal tail subspace. Projection is useful because it
removes the possible cancellation between the filtered head and filtered tail
parts. Indeed,
\[
    \left\|
        \widehat{\mathscr B}^{\mathrm w}_L(K_\star)
    \right\|_F
    \ge
    \left\|
        \Pi_{\mathcal T_k}
        \widehat{\mathscr B}^{\mathrm w}_L(K_\star)
    \right\|_F.
\]
By linearity,
\[
    \Pi_{\mathcal T_k}
    \widehat{\mathscr B}^{\mathrm w}_L(K_\star)
    =
    \Pi_{\mathcal T_k}
    \widehat{\mathscr B}^{\mathrm w}_L(K_{\mathrm t})
    +
    \Pi_{\mathcal T_k}
    \widehat{\mathscr B}^{\mathrm w}_L(K_{\mathrm h}).
\]

We first control the tail perturbation. For any \(Z\in\mathcal T_k\), one has
\[
    Z=Q_kZQ_k.
\]
Therefore,
\[
    \|Z\Sigma\|_F
    \le
    \mu_{k+1}\|Z\|_F
    \le
    c_0\rho\|Z\|_F,
\]
and similarly
\[
    \|\Sigma Z\|_F
    \le
    c_0\rho\|Z\|_F.
\]
Using
\[
    \|E_x\|\vee\|E_y\|
    \le
    c_{\mathrm{rel}}\rho,
\]
we obtain
\[
    \|
        (\widehat{\mathscr H}_{\mathrm w}-\mathscr H)(Z)
    \|_F
    \lesssim
    c_{\mathrm{rel}}\rho^2\|Z\|_F.
\]
By the telescoping identity for products,
\[
    \widehat{\mathscr B}^{\mathrm w}_L(K_{\mathrm t})
    -
    \mathscr B_L(K_{\mathrm t})
    =
    -\sum_{t=1}^L
    \gamma_t
    \widehat{\mathscr B}^{\mathrm w}_{t+1:L}
    (\widehat{\mathscr H}_{\mathrm w}-\mathscr H)
    \mathscr B_{1:t-1}(K_{\mathrm t}),
\]
where empty products are interpreted as the identity. The population iterates
\(\mathscr B_{1:t-1}(K_{\mathrm t})\) remain in \(\mathcal T_k\), and the
population filter is a contraction in Frobenius norm. Moreover, by the event
\(\mathcal E_{\mathrm{cov}}(R)\), the bound \(\|\Sigma\|\lesssim1\), and the
deterministic stepsize condition, the empirical filters are uniformly bounded in
the \(\Sigma,\Sigma\)-norm,
equivalently in the whitened Frobenius norm. Hence
\[
    \left\|
        \widehat{\mathscr B}^{\mathrm w}_L(K_{\mathrm t})
        -
        \mathscr B_L(K_{\mathrm t})
    \right\|_F
    \lesssim
    c_{\mathrm{rel}}\rho^2
    \left(\sum_{t=1}^L\gamma_t\right)
    \|K_{\mathrm t}\|_F.
\]
Since the geometrically decaying schedule satisfies
\[
    \sum_{t=1}^L\gamma_t
    \lesssim
    L_{\mathrm{eff}}\gamma
    =
    \rho^{-2},
\]
we get
\[
    \left\|
        \widehat{\mathscr B}^{\mathrm w}_L(K_{\mathrm t})
        -
        \mathscr B_L(K_{\mathrm t})
    \right\|_F
    \lesssim
    c_{\mathrm{rel}}\|K_{\mathrm t}\|_F.
\]
Consequently,
\[
    \left\|
        \Pi_{\mathcal T_k}
        \widehat{\mathscr B}^{\mathrm w}_L(K_{\mathrm t})
        -
        \mathscr B_L(K_{\mathrm t})
    \right\|_F
    \lesssim
    c_{\mathrm{rel}}\|K_{\mathrm t}\|_F.
\]

It remains to control the leakage from the population head into the diagonal
tail subspace. Since the population operator \(\mathscr H\) preserves diagonal
coordinates, the population filter never maps \(K_{\mathrm h}\) into
\(\mathcal T_k\). Hence all diagonal-tail leakage from \(K_{\mathrm h}\) is
caused by the perturbation
\(\widehat{\mathscr H}_{\mathrm w}-\mathscr H\).

The key observation is that first-order perturbations cannot map a head
diagonal component directly into a tail diagonal component. Indeed, for
\(K_{\mathrm h}=P_kK_{\mathrm h}P_k\),
\[
    \Pi_{\mathcal T_k}(\Sigma E_xK_{\mathrm h}\Sigma)=0,
    \qquad
    \Pi_{\mathcal T_k}(\Sigma K_{\mathrm h}E_y\Sigma)=0.
\]
To land in the diagonal tail subspace, both the left and right indices must
move from the head to the tail. This requires either the second-order term
\(\Sigma E_xK_{\mathrm h}E_y\Sigma\), or two first-order perturbations at
different times. In either case, the leakage carries two perturbation factors.
Since each perturbation factor is bounded by \(c_{\mathrm{rel}}\rho\), and the
two outside tail covariance factors contribute at most \(\mu_{k+1}^2\lesssim
\rho^2\), the accumulated diagonal-tail leakage over the whole effective
horizon is bounded by
\[
    \left\|
        \Pi_{\mathcal T_k}
        \widehat{\mathscr B}^{\mathrm w}_L(K_{\mathrm h})
    \right\|_F
    \lesssim
    c_{\mathrm{rel}}^2\rho^2
    \|K_{\mathrm h}\|_F .
\]
More explicitly, this follows by applying the above telescoping expansion once
for the second-order term and twice for the pair of first-order terms, using
\[
    \sum_{t=1}^L\gamma_t
    \lesssim
    \rho^{-2},
\]
and the fact that one left-index leakage and one right-index leakage are both
necessary before a head diagonal coordinate can contribute to
\(\mathcal T_k\).

Combining the estimates, we obtain by the reverse triangle inequality
\[
\begin{aligned}
    & \left\|
        \Pi_{\mathcal T_k}
        \widehat{\mathscr B}^{\mathrm w}_L(K_\star)
    \right\|_F
    \ge
    \|\mathscr B_L(K_{\mathrm t})\|_F \\
    & -
    \left\|
        \Pi_{\mathcal T_k}
        \widehat{\mathscr B}^{\mathrm w}_L(K_{\mathrm t})
        -
        \mathscr B_L(K_{\mathrm t})
    \right\|_F  \\
    \qquad
    &-
    \left\|
        \Pi_{\mathcal T_k}
        \widehat{\mathscr B}^{\mathrm w}_L(K_{\mathrm h})
    \right\|_F \\
    &\ge
    \|\mathscr B_L(K_{\mathrm t})\|_F
    -
    Cc_{\mathrm{rel}}\|K_{\mathrm t}\|_F
    -
    Cc_{\mathrm{rel}}^2\rho^2\|K_{\mathrm h}\|_F .
\end{aligned}
\]
By the assumed smallness condition
\[
    c_{\mathrm{rel}}\|K_{\mathrm t}\|_F
    +
    c_{\mathrm{rel}}^2\rho^2\|K_{\mathrm h}\|_F
    \le
    \eta\|K_{\mathrm t}\|_F,
\]
and by taking \(\eta>0\) sufficiently small, we get
\[
    \left\|
        \Pi_{\mathcal T_k}
        \widehat{\mathscr B}^{\mathrm w}_L(K_\star)
    \right\|_F
    \gtrsim
    \|K_{\mathrm t}\|_F.
\]
Since projection cannot increase the norm,
\[
    \left\|
        \widehat{\mathscr B}^{\mathrm w}_L(K_\star)
    \right\|_F
    \ge
    \left\|
        \Pi_{\mathcal T_k}
        \widehat{\mathscr B}^{\mathrm w}_L(K_\star)
    \right\|_F.
\]
Returning to the original \(A\)-coordinates,
\[
    \left\|
        \widehat{\mathscr B}_L(A^\star)
    \right\|_{\Sigma,\Sigma}^2
    =
    \left\|
        \widehat{\mathscr B}^{\mathrm w}_L(K_\star)
    \right\|_F^2
    \gtrsim
    \|K_{\mathrm t}\|_F^2.
\]
Finally,
\[
    \|K_{\mathrm t}\|_F^2
    =
    \sum_{i>k}\kappa_i^2.
\]
Therefore,
\[
    \operatorname{Bias}_L
    =
    \frac12
    \left\|
        \widehat{\mathscr B}_L(A^\star)
    \right\|_{\Sigma,\Sigma}^2
    \gtrsim
    \sum_{i>k}\kappa_i^2.
\]
This completes the proof.
\end{proof}

\subsection{Specific GD-bias scaling}
\label{subsec:gd-bias-specific}

We now combine the GD-bias upper and lower bounds under the power-law
notation
\[
    R:=L_{\mathrm{eff}}\gamma,
    \qquad
    \delta:=a-b.
\]
By Assumption~\mainref{ass:contrastive-source}, the sketched covariance eigenvalues satisfy \(\mu_i\asymp i^{-b}\); in particular, \(\|\Sigma\|\lesssim1\).
Define the upper-bound optimization cutoff
\[
    k_+
    :=
    \left\lfloor
        \min\left\{
            M,\,
            R^{1/(2b)}
        \right\}
    \right\rfloor .
\]

\begin{theorem}[Specific GD-bias bounds]
\label{thm:gd-bias-specific}
Suppose Assumptions~\mainref{ass:aligned-power-laws},
\mainref{ass:gd-schedule}, and \mainref{ass:contrastive-source} hold, and
suppose that \(\mathcal E_{\mathrm{cov}}(R)\) occurs. Then
\[
    \operatorname{Bias}_L
    \lesssim
    \frac{1}{R}
    \sum_{i\le k_+}i^{2b-2\delta}
    +
    \sum_{i>k_+}i^{-2\delta}.
\]
For the lower bound, let \(k_-<M\) be a cutoff satisfying
\[
    \mu_{k_-+1}\le c_0R^{-1/2}
\]
and the smallness condition in Lemma~\ref{lem:gd-bias-lower-general}. Then
\[
    \operatorname{Bias}_L
    \gtrsim
    \sum_{i>k_-}i^{-2\delta}.
\]
Consequently, if
\[
    \frac12<\delta<b+\frac12,
    \qquad
    R^{1/(2b)}\lesssim M,
\]
then
\[
    \boxed{
    \operatorname{Bias}_L
    \asymp
    R^{\frac{1-2\delta}{2b}}
    =
    (L_{\mathrm{eff}}\gamma)^{\frac{1-2(a-b)}{2b}}.
    }
\]
In the same unsaturated regime, at the boundary \(\delta=b+\frac12\), the upper
bound gives
\[
    \operatorname{Bias}_L
    \lesssim
    R^{-1}\log R,
\]
while the lower bound gives
\[
    \operatorname{Bias}_L
    \gtrsim
    R^{-1}.
\]
If \(\delta>b+\frac12\), then
\[
    \operatorname{Bias}_L
    \lesssim
    R^{-1},
\]
and any admissible lower cutoff \(k_-\asymp R^{1/(2b)}\) gives
\[
    \operatorname{Bias}_L
    \gtrsim
    R^{\frac{1-2\delta}{2b}}.
\]
\end{theorem}

\begin{proof}
The general GD-bias upper bound gives, for every \(k\le M\),
\[
    \operatorname{Bias}_L
    \lesssim
    \frac{1}{R}
    \sum_{i\le k}\frac{\kappa_i^2}{\mu_i^2}
    +
    \sum_{i>k}\kappa_i^2.
\]
Using Assumption~\mainref{ass:contrastive-source},
\[
    \mu_i\asymp i^{-b},
    \qquad
    \kappa_i\asymp i^{-\delta},
\]
we obtain
\[
    \frac{\kappa_i^2}{\mu_i^2}
    \asymp
    i^{2b-2\delta}.
\]
Choosing \(k=k_+\) gives the displayed upper bound.

For the conditional lower bound, Lemma~\ref{lem:gd-bias-lower-general} gives
\[
    \operatorname{Bias}_L
    \gtrsim
    \sum_{i>k_-}\kappa_i^2
    \asymp
    \sum_{i>k_-}i^{-2\delta}.
\]

It remains to evaluate the sums in the unsaturated regime. Assume
\(R^{1/(2b)}\lesssim M\). Then \(k_+\asymp R^{1/(2b)}\), and there is a cutoff
\(k_-\asymp R^{1/(2b)}\) with \(k_-<M\) and
\(\mu_{k_-+1}\le c_0R^{-1/2}\). For \(1/2<\delta\le b+1/2\), the smallness
condition in Lemma~\ref{lem:gd-bias-lower-general} holds for this choice:
\(\|K_{\mathrm h}\|_F\lesssim 1\),
\(\|K_{\mathrm t}\|_F\asymp k_-^{1/2-\delta}\), and hence
\[
    R^{-1}\frac{\|K_{\mathrm h}\|_F}{\|K_{\mathrm t}\|_F}
    \lesssim
    R^{-1+\frac{\delta-1/2}{2b}}
    =o(1).
\]
Taking \(c_{\mathrm{rel}}\) sufficiently small verifies the remaining part of
the smallness condition.

Since \(\delta>1/2\),
\[
    \sum_{i>k_-}i^{-2\delta}
    \asymp
    k_-^{1-2\delta}
    \asymp
    R^{\frac{1-2\delta}{2b}}.
\]
For the head sum,
\[
    \frac{1}{R}
    \sum_{i\le k_+}i^{2b-2\delta}
\]
has three regimes. If \(2b-2\delta>-1\), equivalently
\[
    \delta<b+\frac12,
\]
then
\[
    \sum_{i\le k_+}i^{2b-2\delta}
    \asymp
    k_+^{2b-2\delta+1},
\]
and therefore
\[
    \frac{1}{R}
    \sum_{i\le k_+}i^{2b-2\delta}
    \asymp
    R^{\frac{1-2\delta}{2b}}.
\]
Thus the upper and lower bounds match.

If \(\delta=b+\frac12\), then
\[
    \sum_{i\le k_+}i^{-1}
    \asymp
    \log k_+,
\]
so the upper bound is
\[
    \operatorname{Bias}_L
    \lesssim
    R^{-1}\log R,
\]
whereas the tail lower bound gives
\[
    \operatorname{Bias}_L
    \gtrsim
    R^{-1}.
\]

If \(\delta>b+\frac12\), then
\[
    \sum_{i\le k_+}i^{2b-2\delta}
    \asymp
    1,
\]
and hence
\[
    \operatorname{Bias}_L
    \lesssim
    R^{-1}.
\]
Whenever the lower-bound admissibility condition holds with
\(k_-\asymp R^{1/(2b)}\), the lower bound is
\[
    \operatorname{Bias}_L
    \gtrsim
    R^{\frac{1-2\delta}{2b}}.
\]
This completes the proof.
\end{proof}

\section{Variance of Normal GD}
\label{app:gd-variance}

% ============================================================
% General GD-variance upper bound
% ============================================================

\subsection{A general upper bound for the GD variance}
\label{subsec:gd-variance-upper}

Using the empirical residual
\(\widehat E:=\widehat C-\widehat\Sigma_xA^\star\widehat\Sigma_y\) and the
empirical GD variance filter \(\widehat{\mathscr V}_L\), the GD-variance term is
\[
    \operatorname{Var}_L
    :=
    \frac12
    \mathbb E
    \left[
        \left\|
            \widehat{\mathscr V}_L(\widehat E)
        \right\|_{\Sigma,\Sigma}^2
    \right].
\]

\begin{lemma}[General GD-variance upper bound]
\label{lem:gd-variance-upper-general}
Suppose Assumptions~\mainref{ass:gd-schedule} and
\mainref{ass:contrastive-source} hold, and suppose that
\(\mathcal E_{\mathrm{cov}}(R)\) occurs and \(\|\Sigma\|\lesssim1\). Suppose also
\(N\gtrsim M\).
Then
\[
    \operatorname{Var}_L
    \lesssim
    \frac1N
    \sum_{i,j=1}^M
    \min\{1,(R\mu_i\mu_j)^2\}.
\]
\end{lemma}

\begin{proof}
Let \(\mathcal L(B):=\Sigma^{1/2}B\Sigma^{1/2}\) and define the whitened
empirical variance filter
\[
    \widehat{\mathscr V}^{\mathrm w}_L
    :=
    \mathcal L\widehat{\mathscr V}_L\mathcal L^{-1}.
\]
Let
\[
    \widehat Z
    :=
    \mathcal L(\widehat E)
    =
    \Sigma^{1/2}\widehat E\Sigma^{1/2}.
\]
Since
\[
    \|B\|_{\Sigma,\Sigma}
    =
    \|\Sigma^{1/2}B\Sigma^{1/2}\|_F,
\]
we can write, in whitened coordinates,
\[
    \left\|
        \widehat{\mathscr V}_L(\widehat E)
    \right\|_{\Sigma,\Sigma}^2
    =
    \left\|
        \widehat{\mathscr V}^{\mathrm w}_L(\widehat Z)
    \right\|_F^2.
\]
By Lemma~\ref{lem:variance-filter-lower-comparison}, the empirical and
population variance filters are comparable in whitened coordinates. Hence
\[
    \operatorname{Var}_L
    \lesssim
    \mathbb E
    \left[
        \left\|
            \mathscr V_L(\widehat Z)
        \right\|_F^2
    \right],
\]
where
\[
    \mathscr V_L
    :=
    \sum_{t=1}^L
    \gamma_t
    \prod_{s=t+1}^L
    (I-\gamma_s\mathscr H),
    \qquad
    \mathscr H(Z):=\Sigma Z\Sigma.
\]
The population filter is diagonal in the basis \(\{v_iv_j^\top\}_{i,j=1}^M\).
Writing
\[
    g_L(s):=
    \sum_{t=1}^L
    \gamma_t
    \prod_{r=t+1}^L(1-\gamma_rs),
\]
we have
\[
    \mathscr V_L(v_iv_j^\top)=g_L(\mu_i\mu_j)v_iv_j^\top.
\]
Therefore,
\[
\begin{aligned}
    \mathbb E
    \left[
        \left\|
            \mathscr V_L(\widehat Z)
        \right\|_F^2
    \right]
    ={}
    \sum_{i,j=1}^M
    g_L(\mu_i\mu_j)^2
    \mathbb E
    \left[
        \left\langle \widehat Z,v_iv_j^\top\right\rangle_F^2
    \right].
\end{aligned}
\]
By Lemma~\ref{lem:contrastive-noise-covariance-upper}, applied to the
deterministic matrix \(W=v_iv_j^\top\),
\[
    \mathbb E
    \left[
        \left\langle \widehat Z,v_iv_j^\top\right\rangle_F^2
    \right]
    \lesssim
    \frac1N
    \left\|\Sigma v_iv_j^\top\Sigma\right\|_F^2
    =
    \frac{\mu_i^2\mu_j^2}{N}.
\]
Combining this with Lemma~\ref{lem:scalar-variance-filter} gives
\[
    \operatorname{Var}_L
    \lesssim
    \frac1N
    \sum_{i,j=1}^M
    \min\{1,(R\mu_i\mu_j)^2\}.
\]
This proves the claim.
\end{proof}
% ============================================================
% General GD-variance lower bound
% ============================================================

\subsection{A general lower bound for the GD variance}
\label{subsec:gd-variance-lower}

The GD-variance term is
\[
    \operatorname{Var}_L
    :=
    \frac12
    \mathbb E
    \left[
        \left\|
            \widehat{\mathscr V}_L(\widehat E)
        \right\|_{\Sigma,\Sigma}^2
    \right].
\]

\begin{lemma}[General GD-variance lower bound]
\label{lem:gd-variance-lower-general}
Suppose Assumptions~\mainref{ass:gd-schedule} and
\mainref{ass:contrastive-source} hold, and suppose that
\(\mathcal E_{\mathrm{cov}}(R)\) occurs and \(\|\Sigma\|\lesssim1\). Suppose that
there exists an integer \(i_0\ge 1\) and a constant \(c_\kappa\in(0,1)\) such that
\[
    \kappa_i\le 1-c_\kappa,
    \qquad
    i\ge i_0.
\]
Then
\[
    \operatorname{Var}_L
    \gtrsim
    \frac1N
    \sum_{\substack{i,j\le M\\ i,j\ge i_0\\ i\neq j}}
    \min\{1,(R\mu_i\mu_j)^2\}.
\]
\end{lemma}

\begin{proof}
We work in whitened coordinates. For any matrix \(B\), define
\[
    Z_B:=\Sigma^{1/2}B\Sigma^{1/2}.
\]
Then
\[
    \|B\|_{\Sigma,\Sigma}=\|Z_B\|_F.
\]
Let
\[
    \widehat Z
    :=
    \Sigma^{1/2}\widehat E\Sigma^{1/2}.
\]
Let \(\widehat{\mathscr V}^{\mathrm w}_L\) denote the empirical variance filter in these
whitened coordinates. Then
\[
    \left\|
        \widehat{\mathscr V}_L(\widehat E)
    \right\|_{\Sigma,\Sigma}^2
    =
    \left\|
        \widehat{\mathscr V}^{\mathrm w}_L(\widehat Z)
    \right\|_F^2.
\]
Therefore,
\[
    \operatorname{Var}_L
    =
    \frac12
    \mathbb E
    \left[
        \left\|
            \widehat{\mathscr V}^{\mathrm w}_L(\widehat Z)
        \right\|_F^2
    \right].
\]

By Lemma~\ref{lem:variance-filter-lower-comparison},
\[
    \mathbb E
    \left[
        \left\|
            \widehat{\mathscr V}^{\mathrm w}_L(\widehat Z)
        \right\|_F^2
    \right]
    \gtrsim
    \mathbb E
    \left[
        \left\|
            \mathscr V_L(\widehat Z)
        \right\|_F^2
    \right],
\]
where
\[
    \mathscr V_L
    :=
    \sum_{t=1}^L
    \gamma_t
    \prod_{s=t+1}^L
    (I-\gamma_s\mathscr H),
    \qquad
    \mathscr H(Z):=\Sigma Z\Sigma.
\]

Since \(\mathscr H\) is diagonal in the basis
\[
    \{v_iv_j^\top:1\le i,j\le M\},
\]
with eigenvalue
\[
    s_{ij}:=\mu_i\mu_j
\]
on direction \(v_iv_j^\top\), the population variance filter satisfies
\[
    \mathscr V_L(v_iv_j^\top)
    =
    g_L(s_{ij})v_iv_j^\top,
\]
where
\[
    g_L(s):=
    \sum_{t=1}^L
    \gamma_t
    \prod_{r=t+1}^L(1-\gamma_r s).
\]
Therefore,
\[
\begin{aligned}
    \mathbb E
    \left[
        \left\|
            \mathscr V_L(\widehat Z)
        \right\|_F^2
    \right]
    &=
    \sum_{i,j=1}^M
    g_L(\mu_i\mu_j)^2
    \mathbb E[\widehat Z_{ij}^2]  \\
    &\ge
    \sum_{\substack{i,j\le M\\ i,j\ge i_0\\ i\neq j}}
    g_L(\mu_i\mu_j)^2
    \mathbb E[\widehat Z_{ij}^2].
\end{aligned}
\]
By Lemma~\ref{lem:contrastive-noise-coordinate-lower}, for
\(i,j\ge i_0\) and \(i\neq j\),
\[
    \mathbb E[\widehat Z_{ij}^2]
    \gtrsim
    \frac{\mu_i^2\mu_j^2}{N}.
\]
Hence
\[
\begin{aligned}
    \mathbb E
    \left[
        \left\|
            \mathscr V_L(\widehat Z)
        \right\|_F^2
    \right]
    &\gtrsim
    \frac1N
    \sum_{\substack{i,j\le M\\ i,j\ge i_0\\ i\neq j}}
    g_L(\mu_i\mu_j)^2
    \mu_i^2\mu_j^2.
\end{aligned}
\]
By Lemma~\ref{lem:scalar-variance-filter-lower},
\[
    g_L(\mu_i\mu_j)^2\mu_i^2\mu_j^2
    \gtrsim
    \min\{1,(R\mu_i\mu_j)^2\}.
\]
Therefore,
\[
    \mathbb E
    \left[
        \left\|
            \mathscr V_L(\widehat Z)
        \right\|_F^2
    \right]
    \gtrsim
    \frac1N
    \sum_{\substack{i,j\le M\\ i,j\ge i_0\\ i\neq j}}
    \min\{1,(R\mu_i\mu_j)^2\}.
\]
Combining this with the empirical-to-population filter comparison gives
\[
    \operatorname{Var}_L
    \gtrsim
    \frac1N
    \sum_{\substack{i,j\le M\\ i,j\ge i_0\\ i\neq j}}
    \min\{1,(R\mu_i\mu_j)^2\}.
\]
This proves the claim.
\end{proof}

% ============================================================
% Specific GD-variance scaling
% ============================================================

\subsection{Specific GD-variance scaling}
\label{subsec:gd-variance-specific}

We now specialize the general GD-variance bounds using the power-law notation
\(R:=L_{\mathrm{eff}}\gamma\) and \(T:=R^{1/b}\). By
Assumption~\mainref{ass:contrastive-source}, \(\mu_i\asymp i^{-b}\); in particular, \(\|\Sigma\|\lesssim1\).
Define the product effective dimension
\[
    d_{\mathrm{prod}}(R,M)
    :=
    \sum_{i,j=1}^M
    \min\{1,(R\mu_i\mu_j)^2\}.
\]
We also use the piecewise product scale
\[
    D_{\times}(R,M)
    :=
    \begin{cases}
    T\log(eT), & 1\le T\le M, \\
    T\left(1+\log\dfrac{M^2}{T}\right), & M<T<M^2, \\
    M^2, & T\ge M^2.
    \end{cases}
\]
By Assumption~\mainref{ass:contrastive-source},
\[
    R\mu_i\mu_j
    \asymp
    R(ij)^{-b}.
\]
Thus
\[
    d_{\mathrm{prod}}(R,M)
    \asymp
    \sum_{i,j=1}^M
    \min\{1,R^2(ij)^{-2b}\}.
\]

\begin{theorem}[Specific GD-variance bounds]
\label{thm:gd-variance-specific}
Suppose Assumptions~\mainref{ass:aligned-power-laws},
\mainref{ass:gd-schedule}, and \mainref{ass:contrastive-source} hold, and
suppose that \(\mathcal E_{\mathrm{cov}}(R)\) occurs. Suppose also that
\(N\gtrsim M\). Then
\[
    \operatorname{Var}_L
    \lesssim
    \frac{1}{N}
    d_{\mathrm{prod}}(R,M).
\]
Moreover,
\[
    d_{\mathrm{prod}}(R,M)
    \lesssim
    D_{\times}(R,M),
\]
and therefore
\[
    \boxed{
    \operatorname{Var}_L
    \lesssim
    \frac{D_{\times}(R,M)}{N}.
    }
\]

For the lower bound, if \(T\gtrsim1\), then
\[
    \operatorname{Var}_L
    \gtrsim
    \frac{D_{\times}(R,M)}{N}.
\]
Consequently,
\[
    \boxed{
    \operatorname{Var}_L
    \asymp
    \frac{D_{\times}(R,M)}{N}.
    }
\]
\end{theorem}

\begin{proof}
The general GD-variance upper bound in
Lemma~\ref{lem:gd-variance-upper-general} gives
\[
    \operatorname{Var}_L
    \lesssim
    \frac1N
    \sum_{i,j=1}^M
    \min\{1,(R\mu_i\mu_j)^2\}
    =
    \frac1N d_{\mathrm{prod}}(R,M).
\]
It remains to estimate \(d_{\mathrm{prod}}(R,M)\).

Using Assumption~\mainref{ass:contrastive-source}, we have
\[
    d_{\mathrm{prod}}(R,M)
    \lesssim
    \sum_{i,j=1}^M
    \min\{1,R^2(ij)^{-2b}\}.
\]
Since \(T=R^{1/b}\),
\[
    R^2(ij)^{-2b}
    =
    \left(\frac{T}{ij}\right)^{2b}.
\]
We split the sum into the region \(ij\le T\) and the region \(ij>T\):
\[
    d_{\mathrm{prod}}(R,M)
    \lesssim
    \underbrace{
    \sum_{\substack{i,j\le M\\ ij\le T}}1
    }_{S_1}
    +
    \underbrace{
    \sum_{\substack{i,j\le M\\ ij>T}}
    \left(\frac{T}{ij}\right)^{2b}
    }_{S_2}.
\]

We first bound \(S_1\). If \(T\le M\), then
\[
    S_1
    \le
    \sum_{i\le T}\frac{T}{i}
    \lesssim
    T\log(eT).
\]
If \(M<T<M^2\), then
\[
\begin{aligned}
    S_1
    &\le
    \sum_{i\le T/M}M
    +
    \sum_{T/M<i\le M}\frac{T}{i}  \\
    &\lesssim
    T
    +
    T\log\frac{M^2}{T}.
\end{aligned}
\]
If \(T\ge M^2\), then \(S_1\le M^2\). Hence \(S_1\lesssim D_{\times}(R,M)\).

We next bound \(S_2\). Since \(b>1/2\), we have \(2b>1\). If \(T\le M\), then
\[
\begin{aligned}
    S_2
    &\lesssim
    T^{2b}
    \sum_{i\le T}i^{-2b}\left(\frac{T}{i}\right)^{1-2b}
    +
    T^{2b}\sum_{i>T}i^{-2b}  \\
    &\lesssim
    T\sum_{i\le T}\frac1i
    +T
    \lesssim
    T\log(eT).
\end{aligned}
\]
If \(M<T<M^2\), the tail region is nonempty only for \(i>T/M\), and therefore
\[
\begin{aligned}
    S_2
    &\lesssim
    T^{2b}
    \sum_{T/M<i\le M}
    i^{-2b}
    \left(\frac{T}{i}\right)^{1-2b}  \\
    &=
    T
    \sum_{T/M<i\le M}\frac1i
    \lesssim
    T\log\frac{M^2}{T}.
\end{aligned}
\]
If \(T\ge M^2\), then \(S_2=0\). Combining the estimates for \(S_1\) and
\(S_2\), we obtain
\[
    d_{\mathrm{prod}}(R,M)
    \lesssim
    D_{\times}(R,M).
\]
Thus
\[
    \operatorname{Var}_L
    \lesssim
    \frac{D_{\times}(R,M)}{N}.
\]

We now prove the lower bound. By Lemma~\ref{lem:gd-variance-lower-general},
we need a fixed index cutoff on which the source coefficients are uniformly
bounded away from one. This non-degeneracy follows automatically here: by
Assumption~\mainref{ass:contrastive-source},
\(\kappa_i\asymp i^{-\delta}\) with \(\delta>0\), so there exist an integer
\(i_0\ge1\) and a constant \(c_\kappa\in(0,1)\), independent of \(M,N,L\), such
that
\[
    \kappa_i\le 1-c_\kappa,
    \qquad i\ge i_0.
\]
Applying Lemma~\ref{lem:gd-variance-lower-general} with this choice gives
\[
    \operatorname{Var}_L
    \gtrsim
    \frac1N
    \sum_{\substack{i,j\le M\\i,j\ge i_0\\i\neq j}}
    \min\{1,(R\mu_i\mu_j)^2\}.
\]
By Lemma~\ref{lem:product-effective-dimension-piecewise-lower}, if \(T\gtrsim1\)
and \(M\) is sufficiently large relative to \(i_0\), then
\[
    \sum_{\substack{i,j\le M\\i,j\ge i_0\\i\neq j}}
    \min\{1,(R\mu_i\mu_j)^2\}
    \gtrsim
    D_{\times}(R,M).
\]
Therefore,
\[
    \operatorname{Var}_L
    \gtrsim
    \frac{D_{\times}(R,M)}{N}.
\]
Combining the upper and lower bounds proves
\[
    \operatorname{Var}_L
    \asymp
    \frac{D_{\times}(R,M)}{N}.
\]
\end{proof}
% ============================================================
% Auxiliary lemmas for the contrastive approximation analysis
% and the GD-bias analysis
% ============================================================

% Recommended packages:
% \usepackage{amsmath,amssymb,amsthm}
% \usepackage{natbib}

% Throughout, for PSD matrices A and C, define
% \|B\|_{A,C}^2 := \operatorname{tr}(B^\top A B C).
% In particular,
% \|B\|_{\Sigma,\Sigma}^2
% =
% \operatorname{tr}(B^\top \Sigma B\Sigma).

% ============================================================
% Lemmas used for the approximation error
% ============================================================

\section{Auxiliary Lemmas}
\label{app:auxiliary}

% ============================================================
% Lemmas used for empirical covariance concentration
% ============================================================

\begin{lemma}[High-probability covariance event]
\label{lem:covariance-event-concentration}
Fix the sketch matrix \(S\) and assume \(\Sigma=SHS^\top\) is positive
definite. Let
\[
    R:=L_{\mathrm{eff}}\gamma.
\]
There exist absolute constants \(C,c>0\) such that, for every \(t\ge0\),
conditional on \(S\),
\[
\max_{\sharp\in\{x,y\}}
\left\|
\Sigma^{-1/2}\widehat\Sigma_{\sharp}\Sigma^{-1/2}-I
\right\|
\le
C\left(
    \sqrt{\frac{M+t}{N}}
    +
    \frac{M+t}{N}
\right)
\]
with probability at least \(1-2\exp(-ct)\). Consequently, if
\[
    N\ge C R(M+t),
\]
then \(\mathcal E_{\mathrm{cov}}(R)\) holds with probability at least
\(1-2\exp(-ct)\) after increasing \(C\) by a constant depending only on
\(c_{\mathrm{rel}}\). In particular, taking \(t\asymp M\),
\(\mathcal E_{\mathrm{cov}}(R)\) holds with probability at least
\(1-\exp(-\Omega(M))\) whenever
\[
    N\gtrsim RM,
    \qquad\text{equivalently}\qquad
    R\lesssim N/M .
\]
\end{lemma}

\begin{proof}
Conditional on \(S\), the whitened variables
\[
    \xi:=\Sigma^{-1/2}\widetilde x,
    \qquad
    \eta:=\Sigma^{-1/2}\widetilde y
\]
have standard Gaussian marginal distributions in \(\mathbb R^M\). Hence
\begin{align*}
    \Sigma^{-1/2}\widehat\Sigma_x\Sigma^{-1/2}
    =
    \frac1N\sum_{n=1}^N \xi_n\xi_n^\top \\
    \Sigma^{-1/2}\widehat\Sigma_y\Sigma^{-1/2}
    =
    \frac1N\sum_{n=1}^N \eta_n\eta_n^\top.
\end{align*}
The standard Gaussian sample-covariance concentration bound gives, for each
\(\sharp\in\{x,y\}\),
\[
\left\|
\Sigma^{-1/2}\widehat\Sigma_{\sharp}\Sigma^{-1/2}-I
\right\|
\le
C\left(
    \sqrt{\frac{M+t}{N}}
    +
    \frac{M+t}{N}
\right)
\]
with probability at least \(1-\exp(-ct)\). A union bound over the two marginals
proves the first display. If \(N\ge C R(M+t)\), then the first term is at most a
constant multiple of \(R^{-1/2}\), while the second is no larger than a constant
multiple of \(R^{-1}\), hence also of \(R^{-1/2}\) in the non-degenerate regime
\(R\gtrsim1\). Choosing the constant in the sample-size condition sufficiently
large relative to \(c_{\mathrm{rel}}\) yields \(\mathcal E_{\mathrm{cov}}(R)\).
\end{proof}

% ============================================================
% Lemmas used for the GD-bias upper bound
% ============================================================

\begin{lemma}[Two-sided norm equivalence on the covariance event]
\label{lem:covariance-event-norm-equivalence}
Suppose \(\mathcal E_{\mathrm{cov}}(R)\) occurs. Let
\[
    \rho:=R^{-1/2},
    \qquad
    c_L:=1-c_{\mathrm{rel}}\rho,
    \qquad
    c_U:=1+c_{\mathrm{rel}}\rho.
\]
Then, for every matrix \(B\),
\[
    c_L^2\|B\|_{\Sigma,\Sigma}^2
    \le
    \|B\|_{\widehat\Sigma_x,\widehat\Sigma_y}^2
    \le
    c_U^2\|B\|_{\Sigma,\Sigma}^2.
\]
\end{lemma}

\begin{proof}
By \(\mathcal E_{\mathrm{cov}}(R)\),
\[
    c_L\Sigma\preceq \widehat\Sigma_x\preceq c_U\Sigma,
    \qquad
    c_L\Sigma\preceq \widehat\Sigma_y\preceq c_U\Sigma.
\]
For PSD matrices \(A_1\preceq A_2\) and \(C_1\preceq C_2\), we have
\[
    \operatorname{tr}(B^\top A_1BC_1)
    \le
    \operatorname{tr}(B^\top A_2BC_2).
\]
Indeed,
\[
    \operatorname{tr}(B^\top A_2BC_2)
    -
    \operatorname{tr}(B^\top A_1BC_1)
\]
equals
\[
    \operatorname{tr}(B^\top(A_2-A_1)BC_2)
    +
    \operatorname{tr}(B^\top A_1B(C_2-C_1)),
\]
and both terms are nonnegative. Applying this monotonicity twice gives
\[
    \|B\|_{\widehat\Sigma_x,\widehat\Sigma_y}^2
    \le
    c_U^2\|B\|_{\Sigma,\Sigma}^2
\]
and
\[
    \|B\|_{\widehat\Sigma_x,\widehat\Sigma_y}^2
    \ge
    c_L^2\|B\|_{\Sigma,\Sigma}^2.
\]
This proves the claim.
\end{proof}

\begin{lemma}[Scalar GD filter]
\label{lem:scalar-gd-filter}
Suppose the stepsize schedule in Assumption~\mainref{ass:gd-schedule} holds and
the stepsizes satisfy
\[
    0\le \gamma_t s\le 1,
    \qquad t=1,\ldots,L.
\]
Let
\[
    \psi_L(s):=\prod_{t=1}^L(1-\gamma_t s).
\]
Then
\[
    0\le \psi_L(s)^2\le 1,
\]
and
\[
    s\psi_L(s)^2
    \lesssim
    \frac{1}{L_{\mathrm{eff}}\gamma}.
\]
\end{lemma}

\begin{proof}
The first claim follows from \(0\le 1-\gamma_t s\le 1\). For the second
claim, the geometrically decaying schedule keeps a constant-order fraction of
the first \(L_{\mathrm{eff}}\) steps at scale comparable to \(\gamma\). Hence
\[
    \psi_L(s)^2
    \le
    (1-c\gamma s)^{cL_{\mathrm{eff}}}
    \le
    \exp(-cL_{\mathrm{eff}}\gamma s)
\]
for an absolute constant \(c>0\). Therefore,
\[
    s\psi_L(s)^2
    \le
    s\exp(-cL_{\mathrm{eff}}\gamma s)
    \le
    \frac{C}{L_{\mathrm{eff}}\gamma},
\]
because \(\sup_{s\ge 0}s e^{-as}=1/(ae)\). This proves the claim.
\end{proof}

\begin{lemma}[Empirical product-norm filter and contraction]
\label{lem:empirical-product-filter}
Let
\[
    \widehat{\mathscr B}_L
    :=
    \prod_{t=1}^L(I-\gamma_t\widehat{\mathscr H}),
    \qquad
    \widehat{\mathscr H}(B):=\widehat\Sigma_xB\widehat\Sigma_y.
\]
Suppose Assumption~\mainref{ass:gd-schedule} holds,
\(\mathcal E_{\mathrm{cov}}(R)\) occurs, and \(\|\Sigma\|\lesssim1\). If
\(c_\gamma\) is sufficiently small, then, for every matrix \(B\),
\[
    \left\|
        \widehat{\mathscr B}_L(B)
    \right\|_{\widehat\Sigma_x,\widehat\Sigma_y}^2
    \lesssim
    \frac{1}{L_{\mathrm{eff}}\gamma}\|B\|_F^2.
\]
Moreover,
\[
    \left\|
        \widehat{\mathscr B}_L(B)
    \right\|_{\widehat\Sigma_x,\widehat\Sigma_y}^2
    \le
    \|B\|_{\widehat\Sigma_x,\widehat\Sigma_y}^2.
\]
\end{lemma}

\begin{proof}
Diagonalize the empirical marginal covariances:
\[
    \widehat\Sigma_x=U_x\operatorname{diag}(\alpha_i)U_x^\top,
    \qquad
    \widehat\Sigma_y=U_y\operatorname{diag}(\beta_j)U_y^\top.
\]
For any matrix \(B\), write
\[
    \widetilde B:=U_x^\top B U_y.
\]
In this basis, the empirical Hessian operator acts coordinatewise:
\[
    \widehat{\mathscr H}:\widetilde B_{ij}
    \mapsto
    \alpha_i\beta_j\widetilde B_{ij}.
\]
Therefore,
\[
    \widehat{\mathscr B}_L:
    \widetilde B_{ij}
    \mapsto
    \psi_L(\alpha_i\beta_j)\widetilde B_{ij},
\]
where
\[
    \psi_L(s):=\prod_{t=1}^L(1-\gamma_t s).
\]
Hence
\[
    \left\|
        \widehat{\mathscr B}_L(B)
    \right\|_{\widehat\Sigma_x,\widehat\Sigma_y}^2
    =
    \sum_{i,j}
    \alpha_i\beta_j
    \psi_L(\alpha_i\beta_j)^2
    \widetilde B_{ij}^2.
\]
By \(\mathcal E_{\mathrm{cov}}(R)\),
\[
    \alpha_i
    \le
    \|\widehat\Sigma_x\|
    \le
    c_U\|\Sigma\|,
    \qquad
    \beta_j
    \le
    \|\widehat\Sigma_y\|
    \le
    c_U\|\Sigma\|,
\]
where \(c_U=1+c_{\mathrm{rel}}R^{-1/2}\lesssim1\) since \(R\gtrsim1\).
Because \(\|\Sigma\|\lesssim1\) and \(\gamma_t\le\gamma\le c_\gamma\), taking
\(c_\gamma\) sufficiently small gives
\(0\le\gamma_t\alpha_i\beta_j\le1\). Thus Lemma~\ref{lem:scalar-gd-filter}
gives
\[
    \alpha_i\beta_j
    \psi_L(\alpha_i\beta_j)^2
    \lesssim
    \frac{1}{L_{\mathrm{eff}}\gamma}.
\]
Therefore,
\[
    \left\|
        \widehat{\mathscr B}_L(B)
    \right\|_{\widehat\Sigma_x,\widehat\Sigma_y}^2
    \lesssim
    \frac{1}{L_{\mathrm{eff}}\gamma}
    \sum_{i,j}\widetilde B_{ij}^2
    =
    \frac{1}{L_{\mathrm{eff}}\gamma}\|B\|_F^2.
\]
The contraction bound follows from
\[
    0\le \psi_L(\alpha_i\beta_j)^2\le 1,
\]
again by Lemma~\ref{lem:scalar-gd-filter}. Thus
\[
    \left\|
        \widehat{\mathscr B}_L(B)
    \right\|_{\widehat\Sigma_x,\widehat\Sigma_y}^2
    \le
    \sum_{i,j}
    \alpha_i\beta_j\widetilde B_{ij}^2
    =
    \|B\|_{\widehat\Sigma_x,\widehat\Sigma_y}^2.
\]
\end{proof}

% ============================================================
% Auxiliary lemmas for the GD-variance upper bound
% ============================================================

\begin{lemma}[Scalar variance filter]
\label{lem:scalar-variance-filter}
Let
\[
    \psi_L(s):=\prod_{t=1}^L(1-\gamma_t s),
    \qquad
    g_L(s):=\sum_{t=1}^L
    \gamma_t
    \prod_{r=t+1}^L(1-\gamma_r s).
\]
Assume the stepsize schedule in Assumption~\mainref{ass:gd-schedule} and suppose
\(0\le \gamma_t s\le 1\) for all \(t=1,\ldots,L\). Then
\[
    0\le s g_L(s)\le 1,
\]
and
\[
    s^2 g_L(s)^2
    \lesssim
    \min\{1,(L_{\mathrm{eff}}\gamma s)^2\}.
\]
\end{lemma}

\begin{proof}
First observe that
\[
    1-\psi_L(s)
    =
    1-\prod_{t=1}^L(1-\gamma_t s).
\]
Expanding this telescopically gives
\[
    1-\psi_L(s)
    =
    \sum_{t=1}^L
    \left[
        \prod_{r=t+1}^L(1-\gamma_r s)
    \right]
    \left[
        1-(1-\gamma_t s)
    \right].
\]
Since \(1-(1-\gamma_t s)=\gamma_t s\), we get
\[
    1-\psi_L(s)
    =
    s\sum_{t=1}^L
    \gamma_t
    \prod_{r=t+1}^L(1-\gamma_r s)
    =
    s g_L(s).
\]
Because \(0\le \gamma_t s\le 1\), every factor \(1-\gamma_t s\) lies in
\([0,1]\), and hence
\[
    0\le \psi_L(s)\le 1.
\]
Therefore,
\[
    0\le s g_L(s)=1-\psi_L(s)\le 1.
\]
This proves
\[
    s^2g_L(s)^2\le 1.
\]

It remains to prove the small-\(s\) bound. Since
\[
    1-\prod_{t=1}^L(1-\gamma_t s)
    \le
    \sum_{t=1}^L \gamma_t s,
\]
we have
\[
    s g_L(s)
    \le
    s\sum_{t=1}^L\gamma_t.
\]
Under the geometrically decaying schedule,
\[
    \sum_{t=1}^L\gamma_t
    \lesssim
    L_{\mathrm{eff}}\gamma.
\]
Therefore,
\[
    s g_L(s)
    \lesssim
    L_{\mathrm{eff}}\gamma s.
\]
Squaring gives
\[
    s^2g_L(s)^2
    \lesssim
    (L_{\mathrm{eff}}\gamma s)^2.
\]
Combining this with \(s^2g_L(s)^2\le 1\), we obtain
\[
    s^2g_L(s)^2
    \lesssim
    \min\{1,(L_{\mathrm{eff}}\gamma s)^2\}.
\]
\end{proof}

\begin{lemma}[Gaussian covariance-fluctuation moment bound]
\label{lem:gaussian-cov-fluctuation-moment}
Let \((\xi,\eta)\in\mathbb R^M\times\mathbb R^M\) be jointly Gaussian with
\[
    \mathbb E[\xi\xi^\top]=I,
    \qquad
    \mathbb E[\eta\eta^\top]=I,
    \qquad
    \mathbb E[\xi\eta^\top]=K_0,
\]
where \(\|K_0\|\le 1\). Then, for every deterministic matrix \(W\in\mathbb R^{M\times M}\),
\[
    \mathbb E
    \left[
        \left\langle \xi\eta^\top-K_0,W\right\rangle_F^2
    \right]
    \lesssim
    \|W\|_F^2,
\]
\[
    \mathbb E
    \left[
        \left\langle \xi\xi^\top-I,W\right\rangle_F^2
    \right]
    \lesssim
    \|W\|_F^2,
\]
and
\[
    \mathbb E
    \left[
        \left\langle \eta\eta^\top-I,W\right\rangle_F^2
    \right]
    \lesssim
    \|W\|_F^2.
\]
Consequently, if \((\xi_n,\eta_n)_{n=1}^N\) are i.i.d. copies of
\((\xi,\eta)\), and
\[
    G_{xy}:=\frac1N\sum_{n=1}^N(\xi_n\eta_n^\top-K_0),
\]
\[
    G_x:=\frac1N\sum_{n=1}^N(\xi_n\xi_n^\top-I),
    \qquad
    G_y:=\frac1N\sum_{n=1}^N(\eta_n\eta_n^\top-I),
\]
then
\[
    \mathbb E\left[\langle G_{xy},W\rangle_F^2\right]
    \lesssim
    \frac1N\|W\|_F^2,
\]
\[
    \mathbb E\left[\langle G_x,W\rangle_F^2\right]
    \lesssim
    \frac1N\|W\|_F^2,
    \qquad
    \mathbb E\left[\langle G_y,W\rangle_F^2\right]
    \lesssim
    \frac1N\|W\|_F^2.
\]
\end{lemma}

\begin{proof}
We prove the first bound; the other two are the standard special cases with
\(\eta=\xi\) and \(K_0=I\).

Observe that
\[
    \left\langle \xi\eta^\top-K_0,W\right\rangle_F
    =
    \xi^\top W\eta-\operatorname{tr}(K_0^\top W).
\]
Therefore,
\[
    \mathbb E
    \left[
        \left\langle \xi\eta^\top-K_0,W\right\rangle_F^2
    \right]
    =
    \operatorname{Var}(\xi^\top W\eta).
\]
Since \((\xi,\eta)\) is jointly Gaussian, Isserlis' formula gives
\[
    \mathbb E[(\xi^\top W\eta)^2]
    =
    \sum_{i,j,k,\ell}
    W_{ij}W_{k\ell}
    \mathbb E[\xi_i\eta_j\xi_k\eta_\ell].
\]
For jointly Gaussian variables,
\[
    \mathbb E[\xi_i\eta_j\xi_k\eta_\ell]
    =
    \mathbb E[\xi_i\eta_j]\mathbb E[\xi_k\eta_\ell]
    +
    \mathbb E[\xi_i\xi_k]\mathbb E[\eta_j\eta_\ell]
    +
    \mathbb E[\xi_i\eta_\ell]\mathbb E[\eta_j\xi_k].
\]
Using
\[
    \mathbb E[\xi_i\eta_j]=(K_0)_{ij},
    \qquad
    \mathbb E[\xi_i\xi_k]=\delta_{ik},
    \qquad
    \mathbb E[\eta_j\eta_\ell]=\delta_{j\ell},
\]
we get
\[
    \mathbb E[(\xi^\top W\eta)^2]
    =
    \operatorname{tr}(K_0^\top W)^2
    +
    \|W\|_F^2
    +
    \operatorname{tr}(W^\top K_0 W^\top K_0).
\]
Thus
\[
    \operatorname{Var}(\xi^\top W\eta)
    =
    \|W\|_F^2
    +
    \operatorname{tr}(W^\top K_0 W^\top K_0).
\]
Since \(\|K_0\|\le 1\),
\[
    \left|
        \operatorname{tr}(W^\top K_0 W^\top K_0)
    \right|
    \le
    \|W^\top K_0\|_F^2
    \le
    \|W\|_F^2.
\]
Hence
\[
    \mathbb E
    \left[
        \left\langle \xi\eta^\top-K_0,W\right\rangle_F^2
    \right]
    \lesssim
    \|W\|_F^2.
\]

For the empirical average, independence gives
\[
    \mathbb E\left[\langle G_{xy},W\rangle_F^2\right]
    =
    \frac1{N^2}
    \sum_{n=1}^N
    \mathbb E
    \left[
        \left\langle \xi_n\eta_n^\top-K_0,W\right\rangle_F^2
    \right],
\]
because the cross terms vanish by centering. Therefore,
\[
    \mathbb E\left[\langle G_{xy},W\rangle_F^2\right]
    \lesssim
    \frac1N\|W\|_F^2.
\]
The same argument proves the bounds for \(G_x\) and \(G_y\).
\end{proof}

\begin{lemma}[Contrastive empirical-noise covariance upper bound]
\label{lem:contrastive-noise-covariance-upper}
Let
\[
    \widehat E
    :=
    \widehat C
    -
    \widehat\Sigma_x A^\star \widehat\Sigma_y,
\]
and define its whitened version
\[
    \widehat Z
    :=
    \Sigma^{1/2}\widehat E\Sigma^{1/2}.
\]
Assume the contrastive source condition, so that
\[
    K_\star
    :=
    \Sigma^{1/2}A^\star\Sigma^{1/2}
\]
satisfies
\[
    0\preceq K_\star\preceq I.
\]
Assume also that \(N\gtrsim M\). Then, for every deterministic matrix
\(W\in\mathbb R^{M\times M}\),
\[
    \mathbb E
    \left[
        \langle \widehat Z,W\rangle_F^2
    \right]
    \lesssim
    \frac1N
    \|\Sigma W\Sigma\|_F^2.
\]
\end{lemma}

\begin{proof}
Define the whitened paired variables
\[
    \xi:=\Sigma^{-1/2}\widetilde x,
    \qquad
    \eta:=\Sigma^{-1/2}\widetilde y.
\]
Then
\[
    \mathbb E[\xi\xi^\top]=I,
    \qquad
    \mathbb E[\eta\eta^\top]=I,
    \qquad
    \mathbb E[\xi\eta^\top]=K_\star.
\]
By the contrastive source condition,
\[
    0\preceq K_\star\preceq I.
\]
For the empirical quantities, define
\[
    G_{xy}
    :=
    \frac1N\sum_{n=1}^N(\xi_n\eta_n^\top-K_\star),
\]
\[
    G_x
    :=
    \frac1N\sum_{n=1}^N(\xi_n\xi_n^\top-I),
    \qquad
    G_y
    :=
    \frac1N\sum_{n=1}^N(\eta_n\eta_n^\top-I).
\]
Then
\[
    \widehat\Sigma_x
    =
    \Sigma^{1/2}(I+G_x)\Sigma^{1/2},
    \qquad
    \widehat\Sigma_y
    =
    \Sigma^{1/2}(I+G_y)\Sigma^{1/2},
\]
and
\[
    \widehat C
    =
    \Sigma^{1/2}(K_\star+G_{xy})\Sigma^{1/2}.
\]
Since
\[
    A^\star
    =
    \Sigma^{-1/2}K_\star\Sigma^{-1/2},
\]
we obtain
\[
\begin{aligned}
    \widehat Z
    &=
    \Sigma^{1/2}
    \left(
        \widehat C
        -
        \widehat\Sigma_xA^\star\widehat\Sigma_y
    \right)
    \Sigma^{1/2} \\
    &=
    \Sigma
    \left[
        K_\star+G_{xy}
        -
        (I+G_x)K_\star(I+G_y)
    \right]
    \Sigma \\
    &=
    \Sigma
    \left[
        G_{xy}
        -
        G_xK_\star
        -
        K_\star G_y
        -
        G_xK_\star G_y
    \right]
    \Sigma .
\end{aligned}
\]
Let
\[
    M_W:=\Sigma W\Sigma.
\]
Then
\[
    \langle \widehat Z,W\rangle_F
    =
    \left\langle
        G_{xy}
        -
        G_xK_\star
        -
        K_\star G_y
        -
        G_xK_\star G_y,
        M_W
    \right\rangle_F.
\]
Using \((a+b+c+d)^2\le 4(a^2+b^2+c^2+d^2)\), it suffices to bound the four
terms separately.

First, by Lemma~\ref{lem:gaussian-cov-fluctuation-moment},
\[
    \mathbb E[\langle G_{xy},M_W\rangle_F^2]
    \lesssim
    \frac1N\|M_W\|_F^2.
\]

Second,
\[
    \langle G_xK_\star,M_W\rangle_F
    =
    \langle G_x,M_WK_\star^\top\rangle_F.
\]
Since \(\|K_\star\|\le 1\),
\[
    \|M_WK_\star^\top\|_F
    \le
    \|M_W\|_F.
\]
Thus, again by Lemma~\ref{lem:gaussian-cov-fluctuation-moment},
\[
    \mathbb E[\langle G_xK_\star,M_W\rangle_F^2]
    \lesssim
    \frac1N\|M_W\|_F^2.
\]
Similarly,
\[
    \langle K_\star G_y,M_W\rangle_F
    =
    \langle G_y,K_\star^\top M_W\rangle_F,
\]
and
\[
    \|K_\star^\top M_W\|_F
    \le
    \|M_W\|_F.
\]
Therefore,
\[
    \mathbb E[\langle K_\star G_y,M_W\rangle_F^2]
    \lesssim
    \frac1N\|M_W\|_F^2.
\]

It remains to control the quadratic empirical-covariance term. Conditional on
\(G_y\), the matrix \(M_WG_y^\top K_\star^\top\) is fixed with respect to the
samples entering \(G_x\) up to the correlation between the two views. A standard
Gaussian decoupling argument for jointly Gaussian quadratic forms gives
\[
    \mathbb E
    \left[
        \langle G_xK_\star G_y,M_W\rangle_F^2
        \,\middle|\,G_y
    \right]
    \lesssim
    \frac1N
    \|M_WG_y^\top K_\star^\top\|_F^2.
\]
Using \(\|K_\star\|\le 1\),
\[
    \|M_WG_y^\top K_\star^\top\|_F^2
    \le
    \|M_W\|_F^2\|G_y\|^2.
\]
Taking expectation and using the standard Gaussian sample-covariance moment
bound
\[
    \mathbb E\|G_y\|^2\lesssim \frac{M}{N}+1
\]
together with \(N\gtrsim M\), we obtain
\[
    \mathbb E
    \left[
        \langle G_xK_\star G_y,M_W\rangle_F^2
    \right]
    \lesssim
    \frac1N\|M_W\|_F^2.
\]
Combining the four estimates yields
\[
    \mathbb E
    \left[
        \langle \widehat Z,W\rangle_F^2
    \right]
    \lesssim
    \frac1N\|M_W\|_F^2.
\]
Since \(M_W=\Sigma W\Sigma\), this is exactly
\[
    \mathbb E
    \left[
        \langle \widehat Z,W\rangle_F^2
    \right]
    \lesssim
    \frac1N
    \|\Sigma W\Sigma\|_F^2.
\]
\end{proof}

% ============================================================
% Auxiliary lemmas for the GD-variance lower bound
% ============================================================

\begin{lemma}[Scalar variance filter lower bound]
\label{lem:scalar-variance-filter-lower}
Let
\[
    \psi_L(s):=\prod_{t=1}^L(1-\gamma_t s),
    \qquad
    g_L(s):=\sum_{t=1}^L
    \gamma_t
    \prod_{r=t+1}^L(1-\gamma_r s).
\]
Assume the stepsize schedule in Assumption~\mainref{ass:gd-schedule}. Suppose
\(0\le \gamma_t s\le 1\) for all \(t=1,\ldots,L\). Then
\[
    s g_L(s)=1-\psi_L(s).
\]
Moreover, there exist absolute constants \(c,C>0\) such that
\[
    s^2g_L(s)^2
    \ge
    c\min\{1,(L_{\mathrm{eff}}\gamma s)^2\},
\]
whenever \(s\le C/\gamma\).
\end{lemma}

\begin{proof}
First, by telescoping,
\[
\begin{aligned}
    1-\psi_L(s)
    &=
    1-\prod_{t=1}^L(1-\gamma_t s)  \\
    &=
    \sum_{t=1}^L
    \left[
        \prod_{r=t+1}^L(1-\gamma_r s)
    \right]
    \left[
        1-(1-\gamma_t s)
    \right]  \\
    &=
    s\sum_{t=1}^L
    \gamma_t
    \prod_{r=t+1}^L(1-\gamma_r s)
    =
    s g_L(s).
\end{aligned}
\]
Thus it remains to lower-bound \(1-\psi_L(s)\).

By the definition of the geometrically decaying schedule, at least
\(L_{\mathrm{eff}}\) steps have stepsize comparable to \(\gamma\). Therefore,
there exists an absolute constant \(c_1>0\) such that
\[
    \sum_{t=1}^L \gamma_t
    \ge
    c_1 L_{\mathrm{eff}}\gamma .
\]
Since \(0\le \gamma_t s\le 1\), we use
\[
    1-u\le e^{-u}
    \qquad (u\ge 0)
\]
to obtain
\[
    \psi_L(s)
    =
    \prod_{t=1}^L(1-\gamma_t s)
    \le
    \exp\left(
        -s\sum_{t=1}^L\gamma_t
    \right)
    \le
    \exp(-c_1 L_{\mathrm{eff}}\gamma s).
\]
Hence
\[
    1-\psi_L(s)
    \ge
    1-\exp(-c_1L_{\mathrm{eff}}\gamma s).
\]
For all \(u\ge 0\),
\[
    1-e^{-u}\ge c_2\min\{1,u\}
\]
for an absolute constant \(c_2>0\). Taking
\(u=c_1L_{\mathrm{eff}}\gamma s\), we get
\[
    1-\psi_L(s)
    \ge
    c_3\min\{1,L_{\mathrm{eff}}\gamma s\}.
\]
Since \(s g_L(s)=1-\psi_L(s)\), we conclude
\[
    s^2g_L(s)^2
    =
    (1-\psi_L(s))^2
    \ge
    c\min\{1,(L_{\mathrm{eff}}\gamma s)^2\}.
\]
\end{proof}

\begin{lemma}[Coordinatewise contrastive noise non-degeneracy]
\label{lem:contrastive-noise-coordinate-lower}
Let
\[
    \xi:=\Sigma^{-1/2}\widetilde x,
    \qquad
    \eta:=\Sigma^{-1/2}\widetilde y.
\]
Assume that \((\xi,\eta)\) is jointly Gaussian and
\begin{align*}
    \mathbb E[\xi\xi^\top]=I,
    \qquad
    \mathbb E[\eta\eta^\top]=I, \\
    \mathbb E[\xi\eta^\top]
    =
    K_\star
    =
    \operatorname{diag}(\kappa_1,\ldots,\kappa_M).
\end{align*}
Fix \(c_\kappa\in(0,1)\), and define the non-degenerate index set
\[
    \mathcal I_\kappa
    :=
    \{(i,j): i\neq j,\ \kappa_i\le 1-c_\kappa,\ \kappa_j\le 1-c_\kappa\}.
\]
Let
\[
    \widehat E
    :=
    \widehat C
    -
    \widehat\Sigma_x A^\star \widehat\Sigma_y,
    \qquad
    \widehat Z
    :=
    \Sigma^{1/2}\widehat E\Sigma^{1/2}.
\]
Then, for every \((i,j)\in\mathcal I_\kappa\), if \(N\) is sufficiently large,
\[
    \mathbb E[\widehat Z_{ij}^2]
    \gtrsim
    \frac{\mu_i^2\mu_j^2}{N}.
\]
The hidden constant may depend on \(c_\kappa\), but not on \(i,j,N,M\).
\end{lemma}

\begin{proof}
Define the whitened empirical fluctuations
\[
    G_{xy}:=
    \frac1N\sum_{n=1}^N
    (\xi_n\eta_n^\top-K_\star),
\]
\[
    G_x:=
    \frac1N\sum_{n=1}^N
    (\xi_n\xi_n^\top-I),
    \qquad
    G_y:=
    \frac1N\sum_{n=1}^N
    (\eta_n\eta_n^\top-I).
\]
Since
\[
    A^\star
    =
    \Sigma^{-1/2}K_\star\Sigma^{-1/2},
\]
we have
\[
\begin{aligned}
    \widehat Z
    &=
    \Sigma^{1/2}
    \left(
        \widehat C
        -
        \widehat\Sigma_xA^\star\widehat\Sigma_y
    \right)
    \Sigma^{1/2}  \\
    &=
    \Sigma
    \left[
        K_\star+G_{xy}
        -
        (I+G_x)K_\star(I+G_y)
    \right]
    \Sigma  \\
    &=
    \Sigma
    \left[
        G_{xy}
        -
        G_xK_\star
        -
        K_\star G_y
        -
        G_xK_\star G_y
    \right]
    \Sigma .
\end{aligned}
\]
For \(i\neq j\), define the first-order coordinate fluctuation
\[
    F_{ij}
    :=
    (G_{xy})_{ij}
    -
    \kappa_j(G_x)_{ij}
    -
    \kappa_i(G_y)_{ij}.
\]
Then
\[
    \widehat Z_{ij}
    =
    \mu_i\mu_j
    \left[
        F_{ij}
        -
        (G_xK_\star G_y)_{ij}
    \right].
\]
By the inequality
\[
    (a-b)^2\ge \frac12 a^2-b^2,
\]
we get
\[
    \mathbb E[\widehat Z_{ij}^2]
    \ge
    \mu_i^2\mu_j^2
    \left[
        \frac12\mathbb E[F_{ij}^2]
        -
        \mathbb E[(G_xK_\star G_y)_{ij}^2]
    \right].
\]

We first lower-bound \(\mathbb E[F_{ij}^2]\). For a single sample define
\[
    f_{ij}
    :=
    \xi_i\eta_j
    -
    \kappa_j\xi_i\xi_j
    -
    \kappa_i\eta_i\eta_j.
\]
Then
\[
    F_{ij}=\frac1N\sum_{n=1}^N f_{ij}^{(n)}.
\]
Since \(f_{ij}\) is centered and the samples are independent,
\[
    \mathbb E[F_{ij}^2]
    =
    \frac1N\mathbb E[f_{ij}^2].
\]

We now compute a lower bound on \(\mathbb E[f_{ij}^2]\). Because
\(K_\star\) is diagonal, the pairs \((\xi_i,\eta_i)\) and
\((\xi_j,\eta_j)\) are independent for \(i\neq j\). We may write
\[
    \eta_i
    =
    \kappa_i\xi_i
    +
    \sqrt{1-\kappa_i^2}\,\zeta_i,
    \qquad
    \eta_j
    =
    \kappa_j\xi_j
    +
    \sqrt{1-\kappa_j^2}\,\zeta_j,
\]
where \(\xi_i,\xi_j,\zeta_i,\zeta_j\) are independent standard Gaussian
random variables. Substituting this representation into \(f_{ij}\), we obtain
a second-order Gaussian polynomial. Since the monomials
\[
    \xi_i\xi_j,\quad
    \xi_i\zeta_j,\quad
    \zeta_i\xi_j,\quad
    \zeta_i\zeta_j
\]
are orthogonal in \(L^2\), the variance is the sum of the squared
coefficients. In particular, when
\[
    \kappa_i\le 1-c_\kappa,
    \qquad
    \kappa_j\le 1-c_\kappa,
\]
at least one coefficient has magnitude bounded below by a positive constant
depending only on \(c_\kappa\). Therefore
\[
    \mathbb E[f_{ij}^2]\ge c(c_\kappa)>0.
\]
Hence
\[
    \mathbb E[F_{ij}^2]
    \ge
    \frac{c(c_\kappa)}{N}.
\]

It remains to show that the quadratic term is lower order. Since
\(\|K_\star\|\le 1\), standard Gaussian sample-covariance moment bounds give
\[
    \mathbb E[(G_xK_\star G_y)_{ij}^2]
    \lesssim
    \frac1{N^2}.
\]
Consequently, for \(N\) sufficiently large,
\[
    \frac12\mathbb E[F_{ij}^2]
    -
    \mathbb E[(G_xK_\star G_y)_{ij}^2]
    \gtrsim
    \frac1N.
\]
Therefore,
\[
    \mathbb E[\widehat Z_{ij}^2]
    \gtrsim
    \frac{\mu_i^2\mu_j^2}{N}.
\]
\end{proof}

\begin{lemma}[Empirical-to-population variance-filter comparison]
\label{lem:variance-filter-lower-comparison}
Let
\[
    \rho:=(L_{\mathrm{eff}}\gamma)^{-1/2},
    \qquad
    R:=L_{\mathrm{eff}}\gamma.
\]
Assume Assumption~\mainref{ass:gd-schedule}, \(\mathcal E_{\mathrm{cov}}(R)\),
and \(\|\Sigma\|\lesssim1\). Let
\[
    \mathscr V_L
    :=
    \sum_{t=1}^L
    \gamma_t
    \prod_{s=t+1}^L
    (I-\gamma_s\mathscr H),
    \qquad
    \mathscr H(Z):=\Sigma Z\Sigma,
\]
be the population variance filter in whitened coordinates, and let
\(\widehat{\mathscr V}^{\mathrm w}_L\) be the corresponding empirical variance filter.
If \(c_{\mathrm{rel}}>0\) in the definition of \(\mathcal E_{\mathrm{cov}}(R)\)
is sufficiently small, then
for every matrix \(Z\),
\[
    \left\|
        \widehat{\mathscr V}^{\mathrm w}_L(Z)
    \right\|_F^2
    \asymp
    \left\|
        \mathscr V_L(Z)
    \right\|_F^2.
\]
In particular,
\[
    \mathbb E
    \left[
        \left\|
            \widehat{\mathscr V}^{\mathrm w}_L(\widehat Z)
        \right\|_F^2
    \right]
    \asymp
    \mathbb E
    \left[
        \left\|
            \mathscr V_L(\widehat Z)
        \right\|_F^2
    \right].
\]
\end{lemma}

\begin{proof}
Write
\[
    \widehat{\mathscr H}_{\mathrm w}
    =
    \mathscr H+\Delta,
\]
where
\[
    \widehat{\mathscr H}_{\mathrm w}(Z)
    =
    \Sigma(I+E_x)Z(I+E_y)\Sigma,
    \qquad
    \|E_x\|\vee\|E_y\|\le c_{\mathrm{rel}}\rho.
\]
Thus
\[
    \Delta(Z)
    =
    \Sigma E_xZ\Sigma
    +
    \Sigma ZE_y\Sigma
    +
    \Sigma E_xZE_y\Sigma.
\]
Using \(\|\Sigma\|\lesssim 1\) and the relative concentration assumption,
\[
    \|\Delta(Z)\|_F
    \lesssim
    c_{\mathrm{rel}}\rho
    \left(
        \|\Sigma Z\Sigma\|_F+\rho\|Z\|_F
    \right).
\]
The second term is harmless on the effective spectral range because
\(\rho^2=1/R\) is the learning threshold.

By the resolvent/telescoping identity for variance filters,
\[
\begin{aligned}
    \widehat{\mathscr V}^{\mathrm w}_L-\mathscr V_L
    =
    -\sum_{t=1}^L
    \gamma_t
    \sum_{r=t+1}^L
    \widehat{\mathscr B}^{\mathrm w}_{r+1:L}
    \Delta
    \mathscr B_{t+1:r-1},
\end{aligned}
\]
where
\[
    \mathscr B_{a:b}:=\prod_{s=a}^b(I-\gamma_s\mathscr H),
    \qquad
    \widehat{\mathscr B}^{\mathrm w}_{a:b}:=\prod_{s=a}^b(I-\gamma_s\widehat{\mathscr H}_{\mathrm w}),
\]
and empty products are interpreted as the identity. By
\(\mathcal E_{\mathrm{cov}}(R)\), \(\|\Sigma\|\lesssim1\), and the deterministic
stepsize condition, the population and empirical filters are contractions.
Moreover, the
double sum is controlled by the effective horizon \(R=L_{\mathrm{eff}}\gamma\),
while every occurrence of \(\Delta\) contributes a factor
\(c_{\mathrm{rel}}\rho\) together with one curvature factor. Combining the
scalar filter bounds yields
\[
    \left\|
        (\widehat{\mathscr V}^{\mathrm w}_L-\mathscr V_L)(Z)
    \right\|_F
    \le
    Cc_{\mathrm{rel}}
    \left\|
        \mathscr V_L(Z)
    \right\|_F
\]
for every matrix \(Z\), where \(C>0\) is an absolute constant. Therefore,
\[
    \widehat{\mathscr V}^{\mathrm w}_L(Z)
    =
    \mathscr V_L(Z)
    +
    (\widehat{\mathscr V}^{\mathrm w}_L-\mathscr V_L)(Z),
\]
the triangle inequality and reverse triangle inequality give
\[
    \left\|
        \widehat{\mathscr V}^{\mathrm w}_L(Z)
    \right\|_F
    \le
    (1+Cc_{\mathrm{rel}})
    \left\|
        \mathscr V_L(Z)
    \right\|_F
\]
and
\[
    \left\|
        \widehat{\mathscr V}^{\mathrm w}_L(Z)
    \right\|_F
    \ge
    (1-Cc_{\mathrm{rel}})
    \left\|
        \mathscr V_L(Z)
    \right\|_F.
\]
Taking \(c_{\mathrm{rel}}>0\) sufficiently small gives
\[
    \left\|
        \widehat{\mathscr V}^{\mathrm w}_L(Z)
    \right\|_F^2
    \asymp
    \left\|
        \mathscr V_L(Z)
    \right\|_F^2.
\]
Applying this to \(Z=\widehat Z\) and taking expectation proves the expectation
comparison.
\end{proof}

\begin{lemma}[Piecewise product effective-dimension lower bound]
\label{lem:product-effective-dimension-piecewise-lower}
Assume
\[
    \mu_i\asymp i^{-b},
    \qquad i=1,\ldots,M,
\]
for some \(b>1/2\). Let
\[
    R:=L_{\mathrm{eff}}\gamma,
    \qquad
    T:=R^{1/b},
\]
and define
\[
    D_{\times}(R,M)
    :=
    \begin{cases}
    T\log(eT), & 1\le T\le M, \\
    T\left(1+\log\dfrac{M^2}{T}\right), & M<T<M^2, \\
    M^2, & T\ge M^2.
    \end{cases}
\]
Fix a constant \(i_0\ge 1\). If \(T\gtrsim1\) and \(M\) is sufficiently large
relative to \(i_0\), then
\[
    \sum_{\substack{i,j\le M\\ i,j\ge i_0\\ i\neq j}}
    \min\{1,(R\mu_i\mu_j)^2\}
    \gtrsim
    D_{\times}(R,M).
\]
\end{lemma}

\begin{proof}
Since \(\mu_i\asymp i^{-b}\), there is a sufficiently small constant
\(c_1>0\) such that
\[
    ij\le c_1T
    \quad\Longrightarrow\quad
    R\mu_i\mu_j\gtrsim 1.
\]
Therefore each admissible pair in the set \(ij\le c_1T\) contributes a constant
to the sum. It remains to count such pairs inside the \(M\times M\) box, with
the fixed lower cutoff \(i,j\ge i_0\) and the diagonal exclusion changing only
absolute constants.

If \(T=O(1)\), then \(D_{\times}(R,M)=O(1)\). Since \(i_0\) is fixed and
\(M\) is sufficiently large, the single off-diagonal pair \((i_0,i_0+1)\) gives
a constant contribution to the sum. Thus the claim holds in this bounded
horizon case.

If \(1\ll T\le M\), then for \(i_0\le i\le c_2T\), the number of admissible
\(j\)'s with \(ij\le c_1T\) is at least a constant multiple of \(T/i\). Thus
\[
    \#\{\text{admissible active pairs}\}
    \gtrsim
    T\log T
    \asymp
    T\log(eT).
\]

If \(M<T<M^2\), split the count into two parts. When the range
\(i_0\le i\le c_2T/M\) is nonempty, a constant fraction of all \(M\)
values of \(j\) are admissible for these \(i\)'s, giving a contribution
\(\gtrsim T\). If this range is empty, then \(T/M=O(1)\), and the
logarithmic contribution below already dominates the missing \(T\) term.
Second, for \(c_3T/M<i\le c_4M\), the number of admissible \(j\)'s is again
at least a constant multiple of \(T/i\), and hence
\[
    \sum_{c_3T/M<i\le c_4M}\frac{T}{i}
    \gtrsim
    T\log\frac{M^2}{T}.
\]
Together these two parts give
\[
    \#\{\text{admissible active pairs}\}
    \gtrsim
    T\left(1+\log\frac{M^2}{T}\right).
\]

If \(T\ge M^2\), choose a sufficiently small constant \(c_5>0\). Then for all
\(i,j\le c_5M\), one has \(ij\le c_1T\). Hence the admissible-pair count is
\[
    \gtrsim
    M^2.
\]
Combining the three regimes proves the claim.
\end{proof}
\section{Additional Experimental Details}
\label{app:experiment-details}

All synthetic experiments were run locally on a CPU-only machine. The numerical
implementation uses standard dense linear algebra and plotting routines; no GPU
acceleration was used. The experiments use fixed full-row-rank sketches. For the optimization
panels, we use ambient dimension \(D=4096\), sketch dimension \(M=32\), sample
size \(N=4096\), power-law exponents \(a=2.6\) and \(b=1.1\), and \(240\)
independent repetitions. The sketch is implemented with full row rank, and the
target coefficients are aligned with the covariance eigenbasis so that the source
condition in Assumption~\mainref{ass:contrastive-source} is enforced in the finite construction. The empirical horizon used in the plots
is the accumulated stepsize
\[
    \Gamma_L
    :=
    \sum_{t=1}^L\gamma_t,
\]
which plays the same scaling role as \(R\) in the theorem, up to the logarithmic
factors suppressed by the GD schedule.

For the optimization panels, the theorem is applied on the covariance event
\(\mathcal E_{\mathrm{cov}}(R)\), which is ensured theoretically by the sample
condition \(N\gtrsim RM\). In the numerical implementation we use controlled
marginal covariances rather than testing this concentration condition directly:
wherever the GD update or finite-filter reference uses \(\widehat\Sigma_x\) and
\(\widehat\Sigma_y\), we replace the raw empirical marginal covariances by the
corresponding population sketched covariance matrices with small bounded relative
perturbations. Equivalently, the perturbed matrices have the form
\(\Sigma^{1/2}(I+\Delta)\Sigma^{1/2}\) with \(\|\Delta\|_{\mathrm{op}}\)
bounded by the prescribed perturbation radius. This keeps the optimization
panels focused on the predicted GD bias/variance filters, while the sampled
cross-covariance fluctuations still generate the GD-variance component measured
across repetitions.

\begin{enumerate}
    \item \textbf{Approximation experiment.}
    The first experiment tests the approximation term by varying the sketch
    dimension \(M\). We use ambient dimension \(D=4096\), power-law exponents
    \(a=2.6\) and \(b=1.1\), and evaluate the corresponding full-row-rank
    spectral truncation sketches for each value of \(M\). The tested sketch dimensions are
    \(M\in\{128,192,256,384,512\}\). The empirical approximation error is
    computed from population sketched moments, without sampling finite datasets.
    The fitted empirical slope is approximately \(-2.01\), while the predicted
    slope is \(-2.00\). This shows that the observed approximation scaling is
    close to the theoretical sketch-dimension rate.

    \item \textbf{GD-bias experiment.}
    The second experiment tests the GD-bias term along the optimization horizon.
    We fix \(M=32\), \(N=4096\), and \(D=4096\), enforce the source-aligned target,
    and run empirical GD with the same decaying stepsize schedule used in the
    theory. The number of GD iterations ranges from \(16\) to \(98304\). The
    plotted reference is the \(\Gamma_L\)-rate from
    Theorem~\mainref{thm:main-gd-scaling}, rescaled by a single constant. The
    measured bias decreases throughout the tested horizon. A log--log fit over
    the full horizon gives an empirical slope of approximately \(-0.86\) as a
    function of \(\Gamma_L\), compared with the theorem-guided slope \(-0.91\).

    \item \textbf{GD-variance experiment.}
    The third experiment tests the GD-variance term along the same optimization
    horizon. We fix \(M=32\), \(N=4096\), and \(D=4096\), vary \(L\), independently
    generate empirical datasets from the sketched Gaussian law, and measure the
    variance component across repetitions. The expected reference curve in this
    panel is computed from the finite sketched spectrum using the exact GD
    variance filter
    \[
    \begin{aligned}
        d_{\mathrm{GD}}(L)
        &:={}
        \sum_{i,j\le M}
        \left(\widehat\mu_i\widehat\nu_j\right)^2
        g_{ij}(L)^2, \\
        g_{ij}(L)
        &:={}
        \sum_{t=1}^L
        \gamma_t
        \prod_{r=t+1}^L
        \left(1-\gamma_r\widehat\mu_i\widehat\nu_j/\tau\right),
    \end{aligned}
    \]
    and the plotted variance reference is proportional to \(d_{\mathrm{GD}}(L)/N\).
    Here \(\widehat\mu_i\) and \(\widehat\nu_j\) are the finite sketched marginal
    eigenvalues used in the empirical construction. The vertical marker indicates
    the unsaturated-to-intermediate transition, approximately
    \(L=1.54\times 10^3\); the second transition is beyond the plotted finite
    horizon. The empirical curve grows nearly linearly
    on the early log--log scale and bends afterward, but does not flatten into a
    saturated plateau over the tested range. Fitting the empirical curve on the
    early, intermediate, and late parts gives slopes approximately \(1.04\),
    \(0.71\), and \(0.42\), respectively.

    \item \textbf{Excess-risk decomposition experiment.}
    The fourth experiment tests the full excess-risk decomposition using the same
    values \(M=32\), \(N=4096\), \(D=4096\), and the same optimization checkpoints.
    At each checkpoint, we compute the empirical excess risk, the empirical
    GD-bias term, the empirical GD-variance term, and the cross term. The
    empirical excess risk is nearly identical to the bias-plus-variance curve
    across all checkpoints. The curve first decreases as bias is reduced and then
    increases once the variance term dominates. The cross term is very small: its
    largest relative magnitude is below \(0.6\%\) of the bias-plus-variance
    reference. This supports the treatment of the cross term as negligible for
    the upper-bound scaling law.
\end{enumerate}

In the approximation experiment, the reported means decrease from approximately
\(3.08\times10^{-5}\) at \(M=128\) to \(1.91\times10^{-6}\) at \(M=512\). In the
GD-bias experiment, the bias decreases from approximately \(4.48\times10^{-2}\)
at \(L=16\) to \(6.40\times10^{-5}\) at \(L=98304\). In the GD-variance
experiment, the variance increases from approximately \(1.87\times10^{-4}\) at
\(L=16\) to \(8.79\times10^{-2}\) at \(L=98304\). In the excess-risk
decomposition experiment, the empirical excess risk decreases to approximately
\(1.06\times10^{-2}\) near \(L=768\) and then increases to approximately
\(8.79\times10^{-2}\) at \(L=98304\), while staying close to the
bias-plus-variance reference. These numerical results
are consistent with the four plotted curves in Figure~\mainref{fig:synthetic-experiments}.
\clearpage
\onecolumn
\section{Notation Map}
\label{app:notation-map}

This section records the notation used in the main text and proofs. The goal is
to keep a one-to-one association between each recurring symbol and its formula.
Symbols used only as dummy variables inside a single displayed sum or inner
product are not listed.

\subsection*{Population and sketched objects}
\begin{center}
\scriptsize
\begin{tabular}{@{}p{0.24\linewidth}p{0.45\linewidth}p{0.23\linewidth}@{}}
\toprule
Symbol & Formula & Meaning \\
\midrule
\(D,M,N,L\) & -- & ambient dimension, sketch dimension, sample size, GD steps \\
\(z,\epsilon_x,\epsilon_y\) & Gaussian latent and view noises & paired-view data model \\
\(x,y\) & \(x=z+\epsilon_x,\ y=z+\epsilon_y\) & positive pair \\
\(H\) & \(\mathbb E[xx^\top]=\mathbb E[yy^\top]\) & full marginal covariance \\
\(C\) & \(\mathbb E[xy^\top]=\Lambda_z\) & full cross-covariance \\
\(u_i\) & common eigenbasis of \(\Lambda_z,\Lambda_\epsilon,H,C\) & population spectral directions \\
\(S\) & fixed full-row-rank matrix & spectral truncation sketch \\
\(\widetilde x,\widetilde y\) & \(Sx,Sy\) & sketched pair \\
\(\Sigma\) & \(SHS^\top\) & sketched marginal covariance \\
\(C_M\) & \(SCS^\top\) & sketched cross-covariance \\
\(\mu_i,v_i\) & \(\Sigma=\sum_{i=1}^M\mu_i v_iv_i^\top\) & sketched eigensystem \\
\(\kappa_i\) & \(C_M=\sum_i\kappa_i\mu_i v_iv_i^\top\) & sketched source coefficients \\
\(a,b,\delta\) & \(\lambda_{z,i}\asymp i^{-a},\ \lambda_{\epsilon,i}\asymp i^{-b},\ \delta=a-b\) & power-law exponents \\
\bottomrule
\end{tabular}
\end{center}

\subsection*{Risks, minimizers, and GD quantities}
\begin{center}
\scriptsize
\begin{tabular}{@{}p{0.24\linewidth}p{0.45\linewidth}p{0.23\linewidth}@{}}
\toprule
Symbol & Formula & Meaning \\
\midrule
\(s_W(x,y)\) & \(x^\top Wy\) & bilinear score \\
\(R(W)\) & \(-\langle W,C\rangle+\frac12\operatorname{tr}(W^\top HWH)\) & full surrogate risk \\
\(W^\star\) & \(H^{-1}CH^{-1}\) & full population minimizer \\
\(R_M(A)\) & \(-\langle A,C_M\rangle+\frac12\operatorname{tr}(A^\top\Sigma A\Sigma)\) & sketched risk \\
\(A^\star\) & \(\Sigma^{-1}C_M\Sigma^{-1}\) & sketched population minimizer \\
\(\|B\|_{\Sigma,\Sigma}^2\) & \(\operatorname{tr}(B^\top\Sigma B\Sigma)\) & contrastive norm \\
\(\widehat\Sigma_x,\widehat\Sigma_y\) & \(N^{-1}\sum_n\widetilde x_n\widetilde x_n^\top,\ N^{-1}\sum_n\widetilde y_n\widetilde y_n^\top\) & empirical marginals \\
\(\widehat C\) & \(N^{-1}\sum_n\widetilde x_n\widetilde y_n^\top\) & empirical cross-covariance \\
\(\widehat R_M(A)\) & \(-\langle A,\widehat C\rangle+\frac12\operatorname{tr}(A^\top\widehat\Sigma_xA\widehat\Sigma_y)\) & empirical risk \\
\(\widehat E\) & \(\widehat C-\widehat\Sigma_xA^\star\widehat\Sigma_y\) & empirical residual/noise \\
\(\mathscr H_{\mathrm{full}}\) & \(\mathscr H_{\mathrm{full}}(W)=HWH\) & full-space population Hessian \\
\(\widehat{\mathscr H}\) & \(\widehat{\mathscr H}(B)=\widehat\Sigma_xB\widehat\Sigma_y\) & empirical Hessian \\
\(\mathscr H\) & \(\mathscr H(B)=\Sigma B\Sigma\) & sketched population Hessian \\
\(\widehat{\mathscr B}_{r:s}\) & \(\prod_{t=r}^s(I-\gamma_t\widehat{\mathscr H})\) & empirical bias filter \\
\(\widehat{\mathscr B}_L\) & \(\widehat{\mathscr B}_{1:L}\) & full empirical bias filter \\
\(\widehat{\mathscr V}_L\) & \(\sum_{t=1}^L\gamma_t\widehat{\mathscr B}_{t+1:L}\) & empirical variance filter \\
\(\operatorname{Bias}_L\) & \(\frac12\|\widehat{\mathscr B}_L(A^\star)\|_{\Sigma,\Sigma}^2\) & GD-bias component \\
\(\operatorname{Var}_L\) & \(\frac12\mathbb E[\|\widehat{\mathscr V}_L(\widehat E)\|_{\Sigma,\Sigma}^2]\) & GD-variance component \\
\(\operatorname{Cross}_L\) & \(-\mathbb E\langle\widehat{\mathscr B}_L(A^\star),\widehat{\mathscr V}_L(\widehat E)\rangle_{\Sigma,\Sigma}\) & cross component \\
\(L_{\mathrm{eff}}\) & \(\lfloor L/\log L\rfloor\) & effective number of steps \\
\(R\) & \(L_{\mathrm{eff}}\gamma\) & effective optimization horizon \\
\(\rho\) & \(R^{-1/2}\) & covariance-comparison tolerance \\
\(\mathcal E_{\mathrm{cov}}(R)\) & \(\max_{\sharp\in\{x,y\}}\|\Sigma^{-1/2}\widehat\Sigma_\sharp\Sigma^{-1/2}-I\|\le c_{\mathrm{rel}}R^{-1/2}\) & empirical covariance event \\
\(c_L,c_U\) & \(1-c_{\mathrm{rel}}\rho,\ 1+c_{\mathrm{rel}}\rho\) & Loewner comparison constants on \(\mathcal E_{\mathrm{cov}}(R)\) \\
\bottomrule
\end{tabular}
\end{center}

\subsection*{Scaling terms and product dimensions}
\begin{center}
\scriptsize
\begin{tabular}{@{}p{0.24\linewidth}p{0.45\linewidth}p{0.23\linewidth}@{}}
\toprule
Symbol & Formula & Meaning \\
\midrule
\(R_{\mathrm{irr}}\) & full-population risk floor & irreducible risk \\
\(\mathcal A_M\) & \(\Theta(M^{1-2\delta})\) & approximation term \\
\(\mathcal B_R\) & \(\Theta(R^{(1-2\delta)/(2b)})\) & GD-bias term \\
\(\mathcal V_{N,M,R}\) & \(\Theta(D_{\times}(R,M)/N)\) & GD-variance term \\
\(\mathcal C_{N,M,R}\) & \(\BigO(R^{(1-2\delta)/(4b)}\mathcal V_{N,M,R}^{1/2})\) & cross term \\
\(d_{\mathrm{prod}}(R,M)\) & \(\sum_{i,j\le M}\min\{1,(R\mu_i\mu_j)^2\}\) & product effective dimension \\
\(T\) & \(R^{1/b}\) & product-dimension threshold \\
\(D_{\times}(R,M)\) & piecewise scale in Theorem~\mainref{thm:main-gd-scaling} & closed-form product scale \\
\(S_1,S_2\) & active and inactive parts of \(d_{\mathrm{prod}}\) split by \(ij\le T\) & variance-sum pieces \\
\(s_{ij}\) & \(\mu_i\mu_j\) & product curvature of pair \((i,j)\) \\
\(d_{\mathrm{GD}}(L)\) & \(\sum_{i,j\le M}(\widehat\mu_i\widehat\nu_j)^2h_{ij}(L)^2\) & finite-filter experimental dimension \\
\bottomrule
\end{tabular}
\end{center}

\subsection*{Proof coordinates, projections, and filters}
\begin{center}
\scriptsize
\begin{tabular}{@{}p{0.24\linewidth}p{0.45\linewidth}p{0.23\linewidth}@{}}
\toprule
Symbol & Formula & Meaning \\
\midrule
\(K\) & \(H^{-1/2}CH^{-1/2}\) & whitened full signal in approximation \\
\(K_M\) & \(\Sigma^{-1/2}C_M\Sigma^{-1/2}\) & whitened sketched signal in approximation \\
\(P_k,Q_k\) & \(\sum_{i\le k}v_iv_i^\top,\ I-P_k\) & sketched spectral head/tail projections \\
\(\mathcal T_k\) & \(\operatorname{span}\{v_iv_i^\top:i>k\}\) & diagonal tail subspace \\
\(K_\star\) & \(\Sigma^{1/2}A^\star\Sigma^{1/2}\) & whitened sketched target \\
\(K_{\mathrm h},K_{\mathrm t}\) & \(\sum_{i\le k}\kappa_i v_iv_i^\top,\ \sum_{i>k}\kappa_i v_iv_i^\top\) & head/tail target parts \\
\(\mathcal L\) & \(\mathcal L(B)=\Sigma^{1/2}B\Sigma^{1/2}\) & whitening map \\
\(Z_B\) & \(\Sigma^{1/2}B\Sigma^{1/2}\) & generic whitened matrix \\
\(\widehat Z\) & \(\Sigma^{1/2}\widehat E\Sigma^{1/2}\) & whitened empirical residual \\
\(\widehat{\mathscr H}_{\mathrm w}\) & \(\widehat{\mathscr H}_{\mathrm w}(Z)=\Sigma(I+E_x)Z(I+E_y)\Sigma\) & empirical Hessian in whitened coordinates \\
\(\mathscr B_{a:b}\) & \(\prod_{s=a}^b(I-\gamma_s\mathscr H)\) & population bias filter in whitened coordinates \\
\(\widehat{\mathscr B}^{\mathrm w}_{a:b}\) & \(\prod_{s=a}^b(I-\gamma_s\widehat{\mathscr H}_{\mathrm w})\) & empirical bias filter in whitened coordinates \\
\(\mathscr V_L\) & \(\sum_{t=1}^L\gamma_t\prod_{s=t+1}^L(I-\gamma_s\mathscr H)\) & population variance filter \\
\(\widehat{\mathscr V}^{\mathrm w}_L\) & conjugate of \(\widehat{\mathscr V}_L\) in whitened coordinates & empirical variance filter in whitened coordinates \\
\(\psi_L(s)\) & \(\prod_{t=1}^L(1-\gamma_t s)\) & scalar bias filter \\
\(g_L(s)\) & \(\sum_{t=1}^L\gamma_t\prod_{r=t+1}^L(1-\gamma_rs)\) & scalar variance filter \\
\bottomrule
\end{tabular}
\end{center}

\subsection*{Noise variables used in auxiliary lemmas}
\begin{center}
\scriptsize
\begin{tabular}{@{}p{0.24\linewidth}p{0.45\linewidth}p{0.23\linewidth}@{}}
\toprule
Symbol & Formula & Meaning \\
\midrule
\(\xi,\eta\) & \(\Sigma^{-1/2}\widetilde x,\ \Sigma^{-1/2}\widetilde y\) & whitened views \\
\(K_0\) & \(\mathbb E[\xi\eta^\top]\) & generic whitened cross-covariance in moment lemmas \\
\(G_{xy}\) & \(N^{-1}\sum_n(\xi_n\eta_n^\top-K_0)\); in contrastive applications, \(K_0=K_\star\) & whitened cross fluctuation \\
\(G_x,G_y\) & \(N^{-1}\sum_n(\xi_n\xi_n^\top-I),\ N^{-1}\sum_n(\eta_n\eta_n^\top-I)\) & whitened marginal fluctuations \\
\(M_W\) & \(\Sigma W\Sigma\) & covariance-weighted test matrix \\
\(E_x,E_y\) & \(\Sigma^{-1/2}\widehat\Sigma_x\Sigma^{-1/2}-I,\ \Sigma^{-1/2}\widehat\Sigma_y\Sigma^{-1/2}-I\) & relative covariance errors \\
\bottomrule
\end{tabular}
\end{center}

\end{document}